\documentclass[10pt]{article}
\usepackage[utf8]{inputenc}
\usepackage[margin=1in]{geometry}
\usepackage{rotating}
\usepackage{booktabs}
\usepackage{array}
\usepackage{amsmath, amsfonts, amssymb}
\usepackage[table,xcdraw]{xcolor}
\usepackage{hyperref}
\usepackage{bbm}
\usepackage{algorithm}
\usepackage[noend]{algpseudocode}
\usepackage{subcaption}
\usepackage{graphbox}
\usepackage{accents}

\usepackage{hhline}

\usepackage[numbers,sort]{natbib}

\usepackage{mathtools}
\mathtoolsset{showonlyrefs}

\usepackage{tabularx,booktabs}
\newcolumntype{Y}{>{\centering\arraybackslash}X}

\usepackage{makecell}

\definecolor{Gray}{gray}{0.8}
\definecolor{LightGray}{gray}{0.95}

\usepackage{amsthm}
\usepackage[normalem]{ulem}
\usepackage{soul}

\newtheorem{theorem}{Theorem} 
\newtheorem{lemma}{Lemma} 
\newtheorem{fact}{Fact} 

\newtheorem{cor}{Corollary} 

\newtheorem{assumption}{Assumption}

\usepackage{chngcntr}
\usepackage{apptools}
\AtAppendix{\counterwithin{theorem}{section}}
\AtAppendix{\counterwithin{prop}{section}}
\AtAppendix{\counterwithin{cor}{section}}
\AtAppendix{\counterwithin{lemma}{section}}

\usepackage{mathtools}

\usepackage{multicol}
\usepackage{multirow}

\def\R{\mathbb{R}}
\def\E{\mathbb{E}}

\newcommand{\C}{\mathcal{C}}

\newcommand{\GG}{\mathbb{G}}

\renewcommand{\P}{\mathbb{P}}
\newcommand{\A}{\mathcal{A}}

\newcommand{\Var}{\mathrm{Var}}

\newcommand{\cd}{\stackrel{d}{\to}}

\newcommand{\Xt}{\tilde{X}}

\newcommand{\Sigmaols}{\Sigma_{\mathrm{OLS}}}
\newcommand{\Sigmaolshat}{\hat{\Sigma}_{\mathrm{OLS}}}
\newcommand{\Xbb}{\mathbb{X}}
\newcommand{\Ybb}{\mathbb{Y}}

\newcommand{\thetaclass}{\hat{\theta}^{\mathrm{class}}}
\newcommand{\Lclass}{L^{\mathrm{class}}}

\newcommand{\thetaPP}{\hat{\theta}^{\mathrm{PP}}}

\newcommand{\ntr}{n_{\mathrm{tr}}}

\newcommand{\jupyter}[1]{\href{#1}{\begingroup
\setbox0=\hbox{\includegraphics[height=1.5em]{figures/jupyter-logo.pdf}}%
\parbox{\wd0}{\box0}\endgroup}}

\DeclareMathOperator*{\argmin}{arg\,min}

\makeatletter
\def\blfootnote{\xdef\@thefnmark{}\@footnotetext}
\makeatother

\title{Cross-Prediction-Powered Inference}

\author{Tijana Zrnic$^*$\quad \quad Emmanuel J. Cand\`es$^\dagger$\\ \\ 
$^*$Department of Statistics and Stanford Data Science\\
$^\dagger$Department of Statistics and Department of Mathematics\\ Stanford University}
\date{}

\begin{document}

\maketitle

\begin{abstract}
While reliable data-driven decision-making hinges on high-quality labeled data, the acquisition of quality labels often involves laborious human annotations or slow and expensive scientific measurements. Machine learning is becoming an appealing alternative as sophisticated predictive techniques are being used to quickly and cheaply produce large amounts of predicted labels; e.g.,~predicted protein structures are used to supplement experimentally derived structures, predictions of socioeconomic indicators from satellite imagery are used to  supplement accurate survey data, and so on. Since predictions are imperfect and potentially biased, this practice brings into question the validity of downstream inferences. We introduce \emph{cross-prediction}: a method for valid inference powered by machine learning. With a small labeled dataset and a large unlabeled dataset, cross-prediction imputes the missing labels via machine learning and applies a form of debiasing to remedy the prediction inaccuracies. The resulting inferences achieve the desired error probability and  
are more powerful than those that only leverage the labeled data. Closely related is the recent proposal of prediction-powered inference \cite{angelopoulos2023prediction}, which assumes that a good pre-trained model is already available. We show that cross-prediction is consistently more powerful than an adaptation of prediction-powered inference in which a fraction of the labeled data is split off and used to train the model. Finally, we observe that cross-prediction gives more stable conclusions than its competitors; its confidence intervals typically have significantly lower variability.
\end{abstract}

\section{Introduction}
\label{sec:intro}

As data-driven decisions fuel progress across science and technology, ensuring that such decisions are reliable is of critical importance. The reliability of data-driven decision-making rests on having access to high-quality data on one hand, and properly accounting for uncertainty on the other.

One frequently discussed issue is that acquiring high-quality data often involves laborious human labeling, or slow and expensive scientific measurements, or overcoming privacy concerns when human subjects are involved. Machine learning offers a promising alternative: sophisticated techniques such as generative modeling and deep neural networks are being used to cheaply produce large amounts of data that would otherwise be too expensive or time-consuming to collect. For example, tools to predict protein structure are supporting wide-ranging research in biology \cite{jumper2021highly, tunyasuvunakool2021highly,bludau2022structural, lin2023evolutionary}; large language models are being used to generate difficult-to-aggregate information about materials that can be used to fight climate change~\cite{chatgptchemistry}; predictions of socioeconomic and environmental conditions based on satellite imagery are being used for downstream policy decisions~\cite{jean2016combining,steele2017mapping,ball2017comprehensive,rolf2021generalizable}. This increasingly common practice, marked by supplementing high-quality data with machine learning outputs, calls for new principles of uncertainty quantification.

In this work we study this problem in the semi-supervised context, where labels are scarce but features are abundant. For example, precise measurements of environmental conditions are difficult to come by but satellite imagery is abundant.
Due to its volume, satellite imagery is routinely used in combination with computer vision algorithms to predict a range of factors on a global scale, including deforestation~\cite{hansen2013high}, poverty rates \cite{jean2016combining}, and population densities~\cite{robinson2017deep}. These predictions provide a compelling substitute for resource-intensive ground-based measurements and surveys.
However, it is crucial to acknowledge that, while promising, the predictions are not infallible. Consequently, downstream inferences that uncritically treat them as ground truth will be invalid.

\begin{figure}[t]
\centering
\includegraphics[width = 0.295\textwidth]{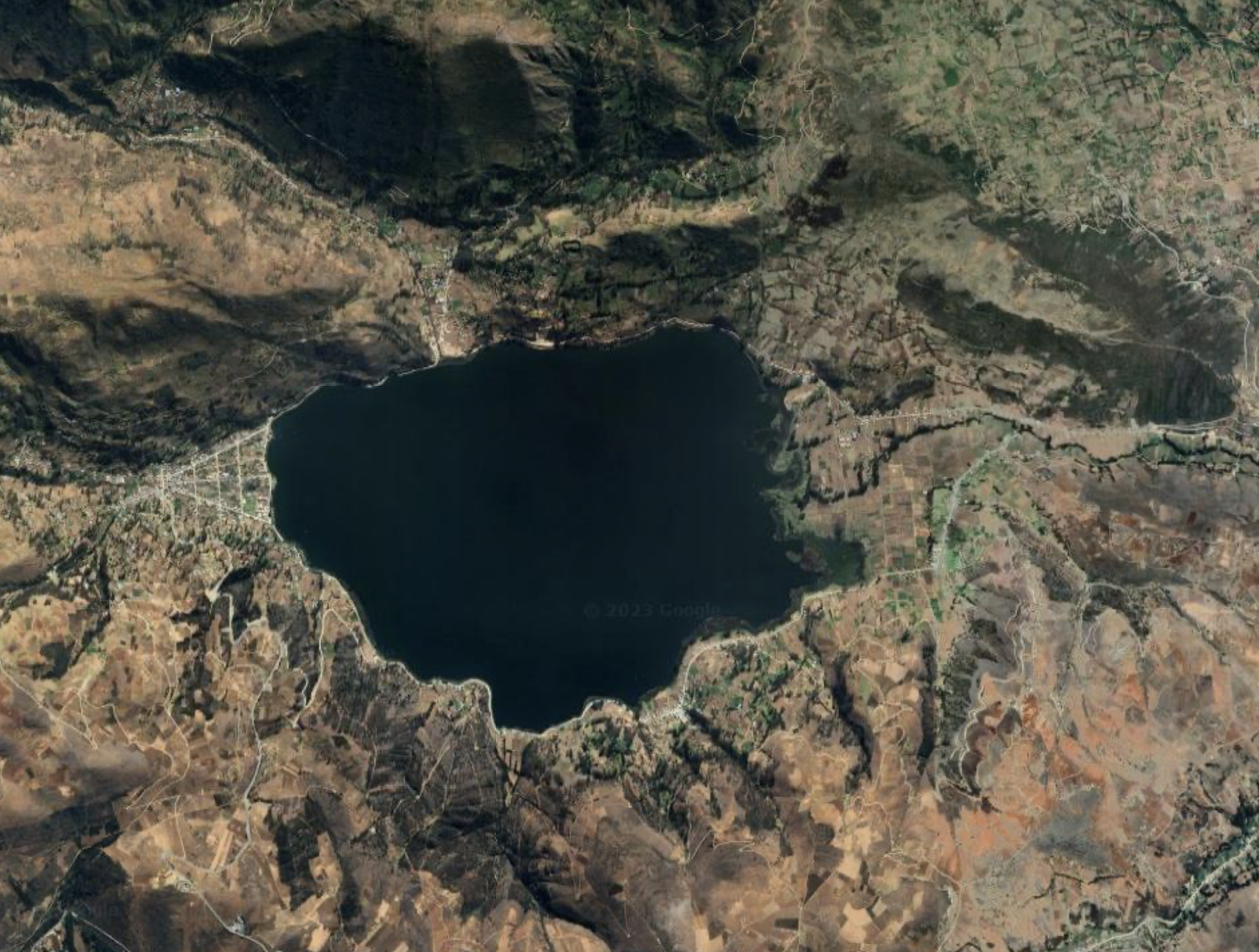}
\includegraphics[width = 0.3075\textwidth]{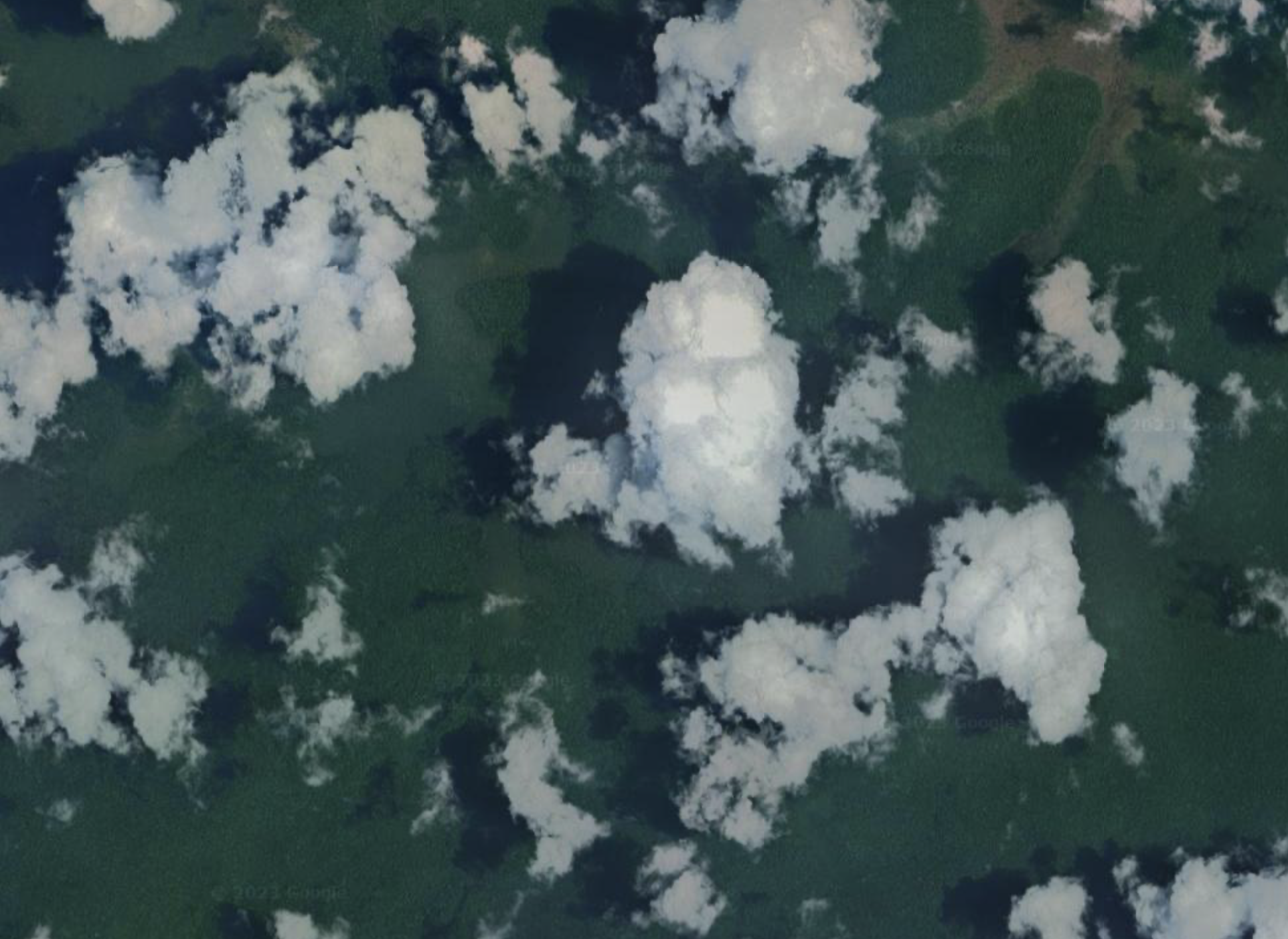}
\includegraphics[width = 0.3\textwidth]{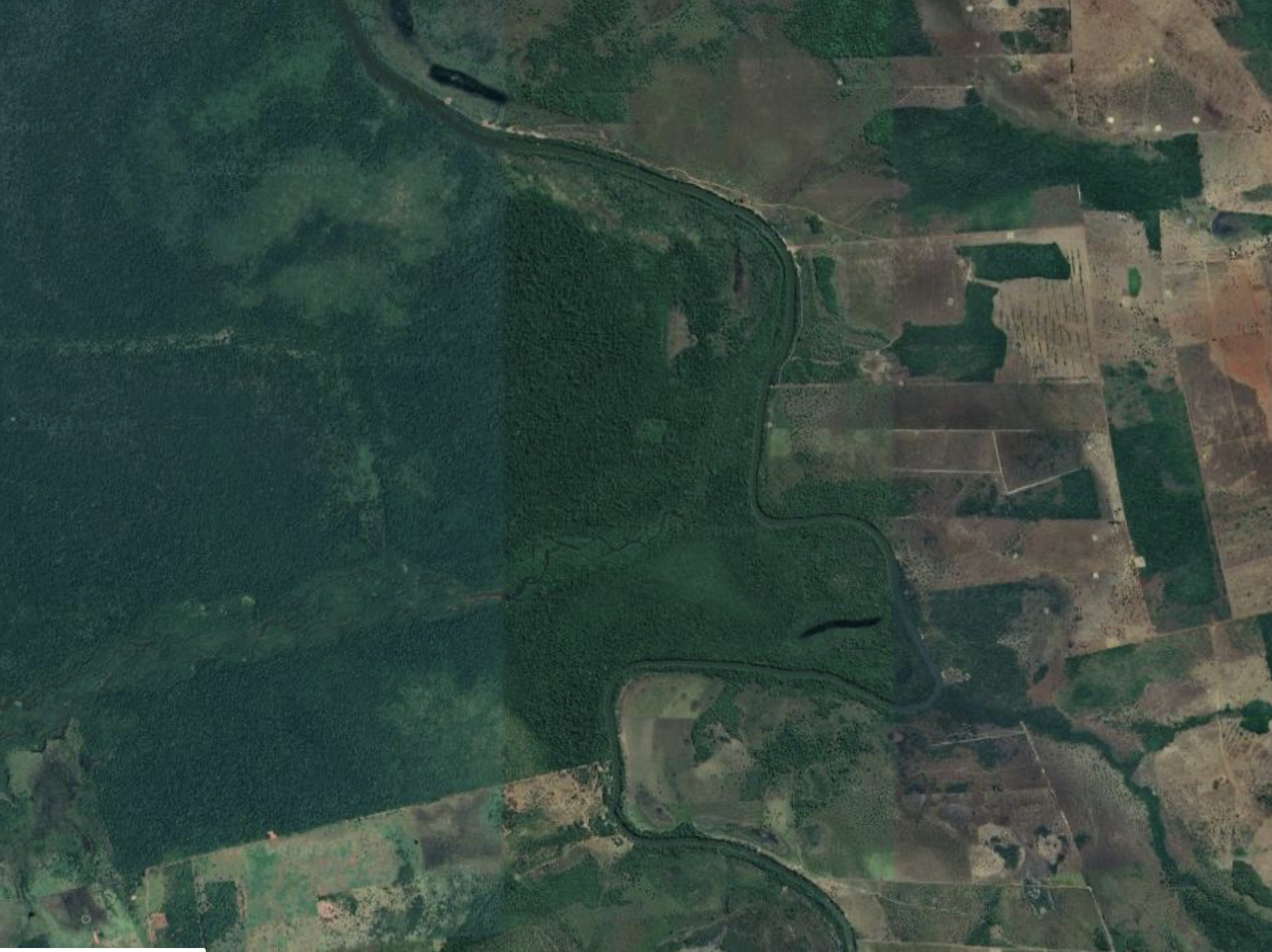}
\caption{Examples of GEE satellite imagery used in the deforestation analysis of \citet{bullock2020satellite}.}
\label{fig:satellite_imagery}
\end{figure}

\begin{figure}[b]
\centering
\includegraphics[width = \textwidth]{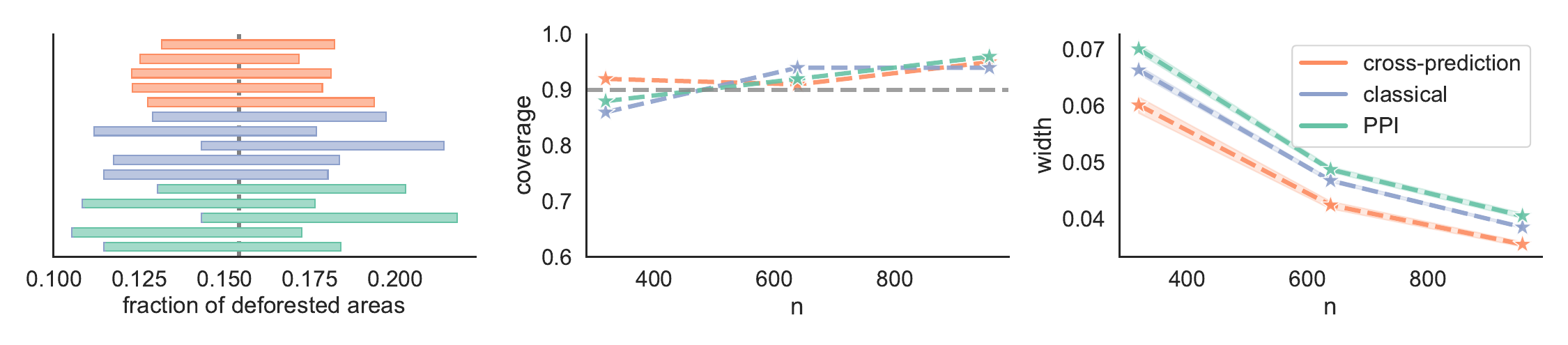}
\caption{\textbf{Estimating the deforestation rate in the Amazon from satellite imagery.} Left: Example intervals constructed by cross-prediction, classical inference, and prediction-powered inference (PPI), for five random splits into labeled and unlabeled data and a fixed number of gold-standard deforestation labels, $n=319$. Middle and right: Coverage and interval width averaged over $100$ random splits into labeled and unlabeled data, for $n\in\{319, 638, 957\}$. The target of inference is the fraction of the Amazon rainforest lost between 2000 and 2015 (gray line in left panel). The target coverage is $90\%$ (gray line in middle panel).
}
\label{fig:deforestation}
\end{figure}

We introduce \emph{cross-prediction}: a broadly applicable method for semi-supervised inference that leverages the power of machine learning while retaining validity. Assume a researcher holds both a small labeled dataset and a large unlabeled dataset, and they seek inference---i.e., a p-value or a confidence interval---about a population-level quantity such as the mean outcome or a regression coefficient. Cross-prediction carefully leverages black-box machine learning to impute the missing labels, resulting in both valid and powerful inferences. The validity is a result of a particular debiasing step; the power is a result of using sophisticated predictive techniques such as deep learning or random forests. We show that the use of black-box predictions on the unlabeled data can lead to a massive improvement in statistical power compared to using the labeled data alone.

Cross-prediction builds upon the recent proposal of \emph{prediction-powered inference} \cite{angelopoulos2023prediction}. Unlike prediction-powered inference, we do not assume that our researcher already has access to a predictive model for imputing the labels. Rather, to apply prediction-powered inference, the researcher would need to use a portion of the labeled data to
either train a model from scratch or fine-tune an off-the-shelf model.
We show that this leads to a suboptimal solution. Consider the following example studied by \citet{angelopoulos2023prediction}.
The goal is to form a confidence interval for the fraction of the Amazon rainforest that was lost between 2000 and 2015. A small number of ``gold-standard'' deforestation labels for certain parcels of land are available, having been collected through field visits \cite{bullock2020satellite}. In addition, satellite imagery is available for the entire Amazon; see Figure \ref{fig:satellite_imagery} for Google Earth Engine (GEE) examples used in the deforestation study of \citet{bullock2020satellite}.
Angelopoulos et al. apply prediction-powered inference after using a portion of the labeled data and a gradient-boosted tree to fine-tune a regression-tree-based predictor of forest cover \cite{sexton2013global}. Our work offers an alternative: we can avoid data splitting and instead apply cross-prediction, still with a gradient-boosted tree, to perform the fine-tuning. By doing so, we significantly reduce the size of the confidence interval, as seen in  Figure \ref{fig:deforestation}. This trend will be consistent throughout our experiments: cross-prediction is more efficient than prediction-powered inference with data splitting. Figure \ref{fig:deforestation} also shows that cross-prediction outperforms ``classical'' inference, which forms a confidence interval based on gold-standard labels only and simply ignores the unlabeled data. 
Additional details about these experiments can be found in Section \ref{sec:deforestation}.

 Another important takeaway from Figure \ref{fig:deforestation} is that cross-prediction gives more stable inferences: the confidence intervals have lower variability than the intervals computed via baseline approaches. Intuitively, classical inference has higher variability due to the smaller sample size, while prediction-powered inference has higher variability due to the arbitrariness in the data split. We will quantify the stability of cross-prediction in our experiments in Section \ref{sec:experiments}, showcasing its superiority across a range of examples, see Table~\ref{table:realdata}.


Our work is also related to the literature known as \emph{semi-supervised inference} \cite{zhang2019semi}. The main difference between existing approaches and our work is that our proposal leverages black-box machine learning methods, allowing for more complicated data modalities (such as high-dimensional imagery) and more sophisticated ways of leveraging the unlabeled data. We elaborate on the relationship to prior work in Section \ref{sec:related_work}.

\section{Problem setup}

We study statistical inference in a semi-supervised setting, where collecting high-quality labels is challenging but feature observations are abundant.
Formally, we have a dataset consisting of $n$ i.i.d.~feature--label pairs, $\{(X_1,Y_1),\dots,(X_n,Y_n)\} \sim \P^n$. 
In addition, we have a dataset consisting of $N$ unlabeled data points, $\{\Xt_1,\dots,\Xt_N \}\sim \P_X^N$, where $\P_X$ denotes the marginal distribution over features according to $\P$. We are most interested in the regime where $N\gg n$, as in the case where feature collection is far cheaper than label collection. 

Our goal is to perform inference on a property $\theta^*(\P)$ of the data-generating distribution $\P$, such as the mean outcome, a quantile of the outcome distribution, or a regression coefficient. Our proposal handles all estimands defined as a solution to an \emph{M-estimation} problem:
\begin{equation}
\label{eq:M-est}
\theta^*(\P) = \argmin_\theta L(\theta), \text{ where } L(\theta):= \E\left[\ell_{\theta}(X,Y)\right] \text{ and } (X, Y) \sim \P,
\end{equation}
for a convex loss function $\ell_\theta$. Here and throughout, $(X,Y)$ denotes a generic sample from $\P$ independent of everything else. All of the aforementioned estimands can be cast in the form~\eqref{eq:M-est}. For example, if the target of inference is the mean outcome, $\theta^*(\P)=\E[Y]$, then $\theta^*(\P)$ minimizes the squared loss:
\begin{align}
\label{eq:sq_loss}
\theta^*(\P) = \argmin_\theta \E[\ell_\theta(Y)] = \argmin_\theta \E[(Y-\theta)^2].
\end{align}
Note that the estimand (and thus the loss) will sometimes only depend on a subset of the features $X$ or only on the outcome $Y$, as in Eq.~\eqref{eq:sq_loss}. Also note that this manuscript focuses on $\theta^*(\P)\in\R^d$ for a \emph{fixed}~$d$. Studying high-dimensional settings---for example, understanding what scaling of $d$ is permitted by the theory---is a valuable direction for future work. Below, we write $\theta^* = \theta^*(\P)$ for short. 

The main question we address is this: how should we leverage the unlabeled data to achieve both valid and powerful inference? Validity alone is an easy target: we can simply dispense with the unlabeled data and find the classical estimator $\thetaclass$, defined as 
\begin{equation}
\label{eq:classical_obj}
\thetaclass = \argmin_\theta  \Lclass(\theta), \text{ where }  \Lclass(\theta) :=  \frac{1}{n} \sum_{i=1}^n \ell_\theta(X_i, Y_i).
\end{equation}
For all standard estimands defined via M-estimation---including means, quantiles, linear regression coefficients---there are off-the-shelf confidence intervals around $\thetaclass$ that cover $\theta^*$ with a desired probability in the large-sample limit \cite[see, e.g.,][]{lehmann2022testing, vandervaart1998asymptotic}. The classical estimator and the corresponding confidence intervals shall be the main comparison points used to evaluate the performance of cross-prediction.

\subsection{Related work}
\label{sec:related_work}

We discuss the relationship between our work and the most closely related technical scholarship.

\textbf{Semi-supervised inference.} Our work falls within the literature known as \emph{semi-supervised inference}~\cite{zhang2019semi}. Most existing work develops methods specialized to particular estimation problems, such as mean estimation~\cite{zhang2019semi, zhang2022high}, quantile estimation \cite{chakrabortty2022semi}, or linear regression \cite{azriel2022semi, tony2020semisupervised}. One exception is the recent work of \citet{song2023general}, who also study general M-estimation. Their approach uses a projection-based correction to the classical loss \eqref{eq:classical_obj} based on simple statistics from the unlabeled data, such as averages of low-degree polynomials of the features. Unlike existing proposals, our approach is based on imputing the missing labels using black-box machine learning methods, allowing for more complicated data modalities and more intricate ways of leveraging the unlabeled data. For example, it is unclear how to apply existing methods when the features $X_i$ are high-dimensional images. We also note that the semi-supervised observation model has been long studied in semi-supervised learning \cite{zhu2005semi,van2020survey,zhu2022introduction}. However, in this literature the goal is prediction, rather than inference.

\textbf{Prediction-powered inference.} The core idea in this paper is to correct imputed predictions, and this derives from the proposal of \emph{prediction-powered inference}~\cite{angelopoulos2023prediction}. However, a key assumption in prediction-powered inference is that, in addition to a labeled and an unlabeled dataset, the analyst is given a good pre-trained machine learning model. We make no such assumption. To apply the theory of prediction-powered inference, our setting would require using a portion of the labeled data for model training and leaving the rest for inference. In contrast, cross-prediction leverages each labeled data point for both model training and inference, leading to a boost in statistical power. The distinction between having and not having a pre-trained model makes a difference even when comparing prediction-powered inference and the classical approach. \citet{angelopoulos2023prediction} do not take into account the data used for model training when comparing the two baselines, because the model is assumed to have been trained before the analysis takes place. This makes sense when considering off-the-shelf models such as AlphaFold. In our comparisons we do take the training data into account. 

\citet{angelopoulos2023ppipp} show a central limit theorem for the prediction-powered estimator, allowing for computational and statistical improvements over the original methods for prediction-powered inference.  Our inferences will be based on a similar central limit theorem for cross-prediction.

\citet{wang2020methods} similarly study inferences based on machine learning predictions. They propose using the labeled data to train a predictor of true outcomes from predicted ones, and then applying the predictor to debias the predictions on the unlabeled data. This algorithm does not come with a formal validity guarantee. \citet{motwani2023valid} conduct a detailed empirical comparison of the method of Wang et al. and prediction-powered inference.

\textbf{Theory of cross-validation.} Cross-prediction is based on a form of cross-fitting. Consequently, our analysis is related to the theoretical studies of cross-validation \cite{dudoit2005asymptotics,austern2020asymptotics,bayle2020cross, kissel2022high, bates2023cross}. In particular, our theory borrows from the analysis of \citet{bayle2020cross}, who prove a central limit theorem and study inference on the cross-validation test error. Our goal, however, is entirely different; we aim to provide inferential guarantees for an estimand $\theta^*$, as defined in Eq.~\eqref{eq:M-est}, in a semi-supervised setting.

\textbf{Semiparametric inference.} Our work is also related to the rich literature on semiparametric inference \cite{levit1976efficiency, hasminskii1979nonparametric, klaassen1987consistent, robinson1988root, bickel1993efficient, newey1994asymptotic, robins1995semiparametric, chernozhukov2018double}, where the goal is to do estimation in the presence of a high-dimensional nuisance parameter. Our debiasing strategy closely resembles doubly-robust estimators \cite{bang2005doubly}, such as the AIPW estimator \cite{robins1994estimation, rotnitzky1998semiparametric}, and one-step estimators \cite{newey1994large}. In this literature the estimand is typically an expected value, such as the average treatment effect. One exception is the work of \citet{jin2023tailored}, who study general M-estimators through a semiparametric lens. The use of cross-fitting is common in that literature as well \cite{chernozhukov2018double,chernozhukov2022locally, chernozhukov2023simple}. While the technical arguments used in our work bear resemblance to those classically used in semiparametric inference, our motivation is different. Our focus is on showcasing how a theoretically-principled use of black-box predictors---neural networks, random forests, and so on---on massive amounts of unlabeled data can boost inference. Since the practice of leveraging unlabeled data through predictions is already prevalent in domains such as remote sensing, our goal is to ground it in statistical theory.

\textbf{Inference with missing data.} Semi-supervised inference can be seen as a special case of the problem of inference with missing data~\cite{rubin1976inference}, where missing information about the labels occurs. Our proposed method bears similarities to multiple imputation \cite{rubin1987multiple,rubin1996multiple,schafer1999multiple} as, at least at a high level, it is based on ``averaging out'' multiple imputed predictions for the labels. However, our method is substantially different from multiple imputation, most notably due to the fact that it incorporates a particular form of debiasing to mitigate prediction inaccuracies.

\textbf{Inference under model misspecification.} Finally, our work relates to a large body of work on inference under model misspecification \cite[e.g.,][]{white1980heteroskedasticity, white1981consequences,buja2019modelsi,buja2019modelsii}. In particular, we do not assume that our predictions follow any ``true'' statistical model, and for parameters $\theta^*$ defined as a regression solution, we do not assume that the regression model is correct. For example, if $\theta^*$ is the solution to a linear regression, we do not assume that the data truly follows a linear model. Like in classical M-estimation, we will show asymptotic normality of our estimator despite the misspecification.

\section{Cross-prediction}

We propose \emph{cross-prediction}---an estimation technique based on a combination of cross-fitting and prediction. The basic idea is to impute labels for the unlabeled data points, and then remove the bias arising from the inaccuracies in the predictions using the labeled data. We give a step-by-step outline of the construction of the cross-prediction estimator. In the following sections we will show how to perform inference with this estimator; that is, how to perform hypothesis tests or construct confidence intervals for $\theta^*$.

\subsection{Cross-prediction for mean estimation}

Before discussing the general case, we consider the problem of mean estimation to gain intuition; the object of inference is simply $\theta^* = \E[Y]$.

We begin by partitioning the labeled dataset into $K$ chunks, $I_1 =\{1,\dots,n/K\}, I_2 =\{n/K + 1,\dots,2n/K\}$, and so on (we assume for simplicity that $n$ is divisible by $K$).\footnote{By removing at most $K-1$ data points, the size of the labeled dataset can be made divisible by $K$. Since in our applications $K$ will typically be equal to $10$, this truncation has a negligible effect.} Here, $K$ is a user-specified number of folds, e.g. $K=10$. Then, as in cross-validation, we train a machine learning model $K$ times, each time training on all data except one fold. Let $\mathcal{A}_{\mathrm{train}}$ denote a possibly randomized training algorithm, which takes as input a dataset of arbitrary size and outputs a predictor of labels from features. Then, for each $j\in[K]$, let $f^{(j)}$ be the model obtained by training on all folds but $I_j$; that is, $f^{(j)} = \mathcal{A}_{\mathrm{train}}(\{(X_i,Y_i)\}_{i\in[n]\setminus I_j})$. We note that $\mathcal{A}_{\mathrm{train}}$ can be quite general; it may or may not treat the training data points symmetrically, and $f^{(j)}$ need not come from a well-defined family of predictors. Rather, $f^{(j)}$ can be any black-box model;  e.g.~a random forest, a gradient-boosted tree, a neural network, and so on. Moreover, $f^{(j)}$ can be trained from scratch or obtained by fine-tuning an off-the-shelf model. Finally, we use the trained models to impute predictions and compute the cross-prediction estimator, defined as:
\begin{equation}
\label{eq:crossfit_mean}
\hat \theta^+= \frac{1}{K  N} \sum_{j=1}^K \sum_{i=1}^N f^{(j)}(\Xt_i) - \frac{1}{n} \sum_{j=1}^K \sum_{i\in I_j} (f^{(j)}(X_i) - Y_i). 
\end{equation}
Intuitively, the first term in Eq.~\eqref{eq:crossfit_mean} is an empirical approximation of the population mean if we treated the predictions as true labels. The second term in Eq.~\eqref{eq:crossfit_mean} serves to debias this heutistic: it subtracts an estimate of the bias between the predicted labels and the true labels.
We note that the estimator \eqref{eq:crossfit_mean} coincides with the mean estimator of \citet{zhang2022high} in the special case where $f^{(j)}$ are linear models, that is, $f^{(j)}(x) = x^\top \beta_j$ for some $\beta_j$. Our analysis applies more broadly, allowing for complex high-dimensional models (e.g., image classifiers).

Observe that the cross-prediction estimator is unbiased, i.e., $\E[\hat \theta^+] = \theta^*$. Indeed, since $i\in I_j$ is not used to train model $f^{(j)}$, we have $\E[f^{(j)}(\Xt_{i'})] = \E[f^{(j)}(X_i)]$ for all $j\in[K],i\in I_j,i'\in[N]$. Applying this identity yields $\E[\hat \theta^+] = \E[Y] = \theta^*$.

The classical estimator is of course the sample mean:
\begin{equation}
\label{eq:classical_mean}
\thetaclass = \frac 1 n \sum_{i=1}^n Y_i,
\end{equation}
which is also unbiased.
Given that both the cross-prediction estimator and the classical estimator are unbiased, it makes sense to ask which one has a lower variance. The main benefit of cross-prediction is that, if the trained models $f^{(j)}$ are reasonably accurate, we expect the variance of the cross-prediction estimator to be lower. To see this, first recall that, typically, $N\gg n$. This means that the first term in $\hat \theta^+$ should have a vanishing variance due to the magnitude of $N$. Therefore,
$$\Var(\hat \theta^+) \approx \Var\left(\frac{1}{n} \sum_{j=1}^K \sum_{i \in I_j} ( f^{(j)}(X_i) -  Y_i )\right).$$
As the sample mean, the remaining term is an average of $n$ terms. However, when the models are accurate, i.e. $f^{(j)}(X_i) \approx Y_i$, we expect $\Var(f^{(j)}(X_i) - Y_i) \ll \Var(Y_i)$.

The closest alternative to the cross-prediction estimator is the prediction-powered estimator~\cite{angelopoulos2023prediction}, that is, its straightforward adaptation to the setup without a pre-trained model. As discussed earlier, prediction-powered inference relies on having a pre-trained model $f$. We can reduce our setting to this case by introducing data splitting: we use the first $\ntr\leq n$ data points from the labeled dataset to train a model $f$ and the rest of the labeled data to compute the prediction-powered estimator:
\begin{equation}
\label{eq:pp_mean}
\thetaPP= \frac{1}{N}  \sum_{i=1}^N f(\Xt_i) - \frac{1}{n-\ntr}  \sum_{i=\ntr+1}^n (f(X_i) - Y_i).
\end{equation}
The prediction-powered estimator is also unbiased: $\E[\thetaPP] = \theta^*$. However, this strategy is potentially wasteful because, after $f$ is trained, the training data is thrown away and not subsequently used for estimation. Cross-prediction uses the data more efficiently, by leveraging each data point for both training and estimation.

\subsection{General cross-prediction}

To introduce the cross-prediction estimator in full generality, recall that we are considering all estimands of the form \eqref{eq:M-est}. As in the case of mean estimation, we split the labeled data into $K$ folds and train a predictive model $f^{(j)}$ on all folds but fold $j\in[K]$. The proposed cross-prediction estimator is defined~as:
\begin{align}
\label{eq:crossfitted_est}
\hat\theta^+ = \argmin_\theta \ &L^+(\theta),
\text{ where } L^+(\theta) := \frac{1}{K N} \sum_{j=1}^K \sum_{i=1}^N \tilde \ell_{\theta,i}^{f^{(j)}} - \frac{1}{n} \sum_{j=1}^K \sum_{i \in I_j} ( \ell_{\theta,i}^{f^{(j)}} -  \ell_{\theta,i} ).
\end{align}
Here, we use the short-hand notation $\tilde \ell_{\theta,i}^{f^{(j)}} :=  \ell_{\theta}(\Xt_i,f^{(j)}(\Xt_i))$, $ \ell_{\theta,i}^{f^{(j)}} :=  \ell_{\theta}(X_i,f^{(j)}(X_i))$, and $ \ell_{\theta,i} :=  \ell_{\theta}(X_i,Y_i)$. The intuition remains the same as before: the first term is an empirical approximation of the population loss if we treated the predictions as true labels, and the second term aims to debias this heuristic. One can verify that the mean estimator in Eq.~\eqref{eq:crossfit_mean} is indeed a special case of the general estimator in Eq.~\eqref{eq:crossfitted_est}, by taking $\ell_\theta$ to be the squared loss, as per Eq.~\eqref{eq:sq_loss}.

The cross-prediction estimator optimizes an unbiased objective, since $\E[L^+(\theta)] = L(\theta)$. This follows because $\E[\ell_\theta (\Xt_{i'}, f^{(j)}(\Xt_{i'}))] = \E[\ell_\theta (X_i, f^{(j)}(X_i))]$ for all $j\in[K],i\in I_j,i'\in[N]$, seeing that $i\in I_j$ is not used to train model $f^{(j)}$.
Furthermore, by the same argument as before, we expect $L^+(\theta)$ to have a lower variance than the classical objective in Eq. \eqref{eq:classical_obj} if $N$ is large and the trained predictors are reasonably accurate.
We note that $L^+(\theta)$ may not be a convex function in general, but solving for $\hat\theta^+$ is tractable in many cases of interest. For example, in the case of means and generalized linear models, $L^+(\theta)$ is convex. 

The prediction-powered estimator is similar to the cross-prediction estimator, but it requires data splitting and does not average over multiple model fits. It is equal to
\begin{align}
\label{eq:thetaPP}
\thetaPP = \argmin_\theta \ &L^{\mathrm{PP}}(\theta), \text{ where } L^{\mathrm{PP}}(\theta) := \frac{1}{N} \sum_{i=1}^N \tilde \ell_{\theta,i}^f - \frac{1}{n-\ntr} \sum_{i =\ntr+1}^n ( \ell_{\theta,i}^f -  \ell_{\theta,i}),
\end{align}
where, as before, $f$ is trained on the first $\ntr$ labeled data points.
The fact that cross-prediction averages the results of multiple model fits allows it to achieve more stable inference. Indeed, in our experiments we will observe that cross-prediction is more stable than prediction-powered inference throughout.

\section{Inference for the mean}
\label{sec:inference_means}

We now discuss inference with the cross-prediction estimator. For simplicity we first look at mean estimation, where $\theta^*=\E[Y]$. We will see that much of the discussion will carry over to general M-estimation problems.

Inference with the cross-prediction estimator in Eq.~\eqref{eq:crossfit_mean} is difficult because the terms being averaged are all dependent through the labeled data. In contrast, the classical estimator in Eq.~\eqref{eq:classical_mean} averages independent terms, allowing for confidence intervals based on the central limit theorem. The prediction-powered estimator in Eq.~\eqref{eq:pp_mean} is
similarly amenable to inference based on the central limit theorem, seeing that all the terms are independent conditional on $f$. In this section we show that, under a relatively mild regularity condition, the cross-prediction estimator likewise satisfies a central limit theorem. This will in turn immediately allow constructing confidence intervals and hypothesis tests for $\theta^*$.

The central limit theorem will require that, as the sample size grows, the predictions concentrate sufficiently rapidly around their expectation. Intuitively, one can think of the condition as requiring that the predictions are sufficiently stable.
While the stability property is difficult to verify for complex black-box models, we empirically observe that inference based on the resulting central limit theorem nevertheless provides the correct coverage. We observe this across different estimation problems, data modalities, sample sizes, and so on.

Our analysis based on stability is inspired by the work of \citet{bayle2020cross}, who study inference on the cross-validation test error, since the inferential challenges in cross-prediction are similar to those in cross-validation. The ultimate goals of the two analyses are, however, entirely different.

Below we state the stability condition.
For all $x$, we define $\bar f(x) := \E[f^{(1)}(x)]$; the ``average'' model $\bar f$ is the predictor we would obtain if we could train many models on independent datasets of size $n-n/K$ and average out their predictions. 

\begin{assumption}
\label{ass:stability}
 We say that the predictions are stable if, as $n\to\infty$,
 $$\sqrt{K \, \Var\left(f^{(1)}(X) - \bar f(X) \mid f^{(1)} \right)} \stackrel{L^1}{\rightarrow} 0.$$
 \end{assumption}


Assumption \ref{ass:stability} requires that the models $f^{(j)}$ converge to their ``average'' model $\bar f$, but there is no assumption that $\bar f$ is by any means well-specified. If the number of folds is fixed (e.g., $K=10$), as we will typically assume, then Assumption \ref{ass:stability} is satisfied if the variance of the difference between the learned predictions $f^{(1)}(X)$ and the average predictions $\bar f(X)$ vanishes at any rate, $\Var\left(f^{(1)}(X) - \bar f(X) \mid f^{(1)} \right) \stackrel{L^1}{\rightarrow} 0$. We expect that any reasonably stable learning algorithm $\A_{\mathrm{train}}$ should satisfy Assumption \ref{ass:stability} (intuitively, any algorithm not too sensitive to perturbing a single data point). Violations of the assumption might arise if the number of folds is allowed to grow, e.g. as in the case of leave-one-out cross-fitting, since then the variance has to tend to zero sufficiently rapidly.

Equipped with Assumption \ref{ass:stability}, we can now state the central limit theorem for cross-prediction.

\begin{theorem}[Cross-prediction CLT for the mean]
\label{thm:clt_mean}
Let $\theta^*$ be the mean outcome, $\theta^* = \E[Y]$.
Suppose that the predictions are stable (Ass. \ref{ass:stability}). Further, assume that $\frac n N$ has a limit, and that  $\bar \sigma^2 = \Var(\bar f(X))$ and $\bar \sigma_\Delta^2 = \Var(\bar f(X) - Y)$ have a nonzero limit. Then,
$$\frac{\sqrt{n}}{\sqrt{\frac n N \bar\sigma^2 + \bar \sigma_\Delta^2}} \left(\hat \theta^+ - \theta^* \right) \cd \mathcal N\left(0,1\right).$$
\end{theorem}

With this, inference on $\theta^*$ is now straightforward as long as we can estimate the asymptotic variance consistently. We will discuss strategies for doing so in Section \ref{sec:variance_est}.

\begin{cor}[Inference for the mean via cross-prediction]
\label{cor:mean_inf}
Let $\theta^*$ be the mean outcome, $\theta^* = \E[Y]$.
Assume the conditions of Theorem \ref{thm:clt_mean}, and suppose that we have estimators $\hat\sigma^2 \stackrel{p}{\to} \bar \sigma^2$ and $\hat\sigma_\Delta^2 \stackrel{p}{\to} \bar \sigma_\Delta^2$. Let
$$\C_\alpha^+ = \left(\hat\theta^+ \pm z_{1- \alpha/2} \frac{\sqrt{\frac n N \hat\sigma^2 + \hat\sigma_\Delta^2}}{\sqrt{n}}\right).$$
Then,
$$\liminf_{n,N} ~\P\left(\theta^* \in \C_\alpha^+\right) \geq 1-\alpha.$$
\end{cor}

\noindent Per standard notation, $z_{1- \alpha/2}$ denotes the $(1-\alpha/2)$-quantile of the standard normal distribution. Corollary~\ref{cor:mean_inf} follows by a direct application of Theorem \ref{thm:clt_mean}, together with Slutsky's theorem.

\section{Inference for general M-estimation}

We generalize the principle introduced in Section \ref{sec:inference_means} to handle arbitrary M-estimation problems. Indeed, the results presented in this section will strictly subsume the results of Section~\ref{sec:inference_means}. 

As in the case of the mean, we will require that the predictions are ``stable'' in an appropriate sense. Naturally, the notion of stability will depend on the loss function used to define the M-estimator.

\begin{assumption}
\label{ass:stability_general}
With $\bar f(\cdot)$ as before, we  say that the predictions are stable if for all $\theta$, as $n\to\infty$,
$$\sqrt{K \, \Var\left(\nabla \ell_{\theta}(X,f^{(1)}(X)) - \nabla \ell_{\theta}(X,\bar f(X) ) \mid f^{(1)} \right)} \stackrel{L^1}{\rightarrow}0.$$
 \end{assumption}

Here, $\Var(\cdot \mid f^{(1)})$ denotes the covariance matrix conditional on $f^{(1)}$. Also, for vectors and matrices, by ``$\stackrel{L^1}{\to} 0$'' we mean convergence in mean to zero element-wise.
Notice that by setting $\ell_\theta(y) = (\theta-y)^2$ to be the squared loss, Assumption \ref{ass:stability_general} reduces to Assumption \ref{ass:stability} in the case of mean estimation. As in the case of Assumption \ref{ass:stability}, Assumption \ref{ass:stability_general} should be interpreted as a stability requirement on $\A_{\mathrm{train}}$. Moreover, there is again no requirement of correct specification of $\bar f$.

We will provide two approaches to inference in this section; which one is more appropriate will depend on the inference problem at hand.

One approach will be based on the characterization of $\theta^*$ as a zero of the gradient of the expected loss,
$\nabla L(\theta^*) = \E[\nabla \ell_{\theta^*}(X,Y)] = 0$,
which follows by the convexity of the loss.
In particular, we will construct a confidence set for $\theta^*$ by finding all $\theta$ accepted by a valid test for the null hypothesis that $\nabla L(\theta) = 0$. Since the test is valid and $\theta^*$ satisfies the null condition, the true solution $\theta^*$ will be excluded with small probability.
The hypothesis test for the population gradient $\nabla L(\theta)$ will follow from a central limit theorem for the gradient of the cross-prediction loss,
$$\nabla L^+(\theta) = \frac{1}{K N} \sum_{j=1}^K \sum_{i=1}^N  \nabla \tilde \ell_{\theta,i}^{f^{(j)}} - \frac{1}{n} \sum_{j=1}^K \sum_{i\in I_j}( \nabla \ell_{\theta,i}^{f^{(j)}} - \nabla \ell_{\theta,i} ).$$

The other approach will be based on showing asymptotic normality of the cross-prediction estimator. For this, we build on the proof of asymptotic normality of the prediction-powered estimator (with a pre-trained model) \cite{angelopoulos2023ppipp}, which in turn builds on classical asymptotic normality of M-estimators \cite{vandervaart1998asymptotic}.
The asymptotic normality will allow forming standard CLT intervals around $\hat\theta^+$.

We implicitly assume mild regularity on the losses $\ell_{\theta}(x,y)$ and $\ell_{\theta}(x,f^{(j)}(x))$, in particular that they are differentiable and locally Lipschitz around $\theta^*$ for all possible $f^{(j)}$ (see Def. A.1 in \cite{angelopoulos2023ppipp}). Our second inference approach will require the usual condition that $\hat\theta^+$ is consistent, $\hat\theta^+ \stackrel{p}{\to} \theta^*$; this is satisfied quite broadly, e.g. when the parameter space is compact or when $L^+(\theta)$ is convex. The latter holds for all generalized linear models, for example. See \cite{vandervaart1998asymptotic,angelopoulos2023ppipp} for further discussion.
 
Theorem \ref{thm:clt_general} states the main technical result of this section, which extends Theorem \ref{thm:clt_mean} to general M-estimation problems.

\begin{theorem}[Cross-prediction CLT]
\label{thm:clt_general}
Suppose that the predictions are stable (Ass. \ref{ass:stability_general}). Further, assume that $\frac n N$ has a limit, and that $\bar \Sigma_\theta = \Var(\nabla \ell_{\theta,i}^{\bar f})$ and $\bar \Sigma_{\Delta,\theta} = \Var(\nabla \ell_{\theta,i}^{\bar f} - \nabla \ell_{\theta,i})$ have a nonzero limit. Denote $\bar V_{\theta} = \frac n N \bar\Sigma_\theta + \bar \Sigma_{\Delta,\theta}$. Then,
$$\sqrt{n} \bar V_{\theta}^{-1/2} \left(\nabla L^+(\theta) - \nabla L(\theta) \right) \cd \mathcal N(0, I).$$
If, additionally, the Hessian $H_{\theta^*} = \nabla^2 L(\theta^*)$ is non-singular, $\hat\theta^+ \stackrel{p}{\to} \theta^*$, and $K$ is constant, then
$$\sqrt{n} \bar \Sigma^{-1/2} \left(\hat\theta^+ - \theta^* \right) \cd \mathcal N(0, I),$$
where $\bar \Sigma = H_{\theta^*}^{-1}\bar V_{\theta^*} H_{\theta^*}^{-1}$. 
\end{theorem}

Theorem \ref{thm:clt_general} immediately yields two methods for computing a confidence set for $\theta^*$, as stated below.


\begin{cor}[Inference via cross-prediction]
\label{cor:confint_general}
Suppose that we have estimators $\hat\Sigma \stackrel{p}{\to} \bar \Sigma$ and $\hat V_{\theta} \stackrel{p}{\to} \bar V_{\theta}$, for all~$\theta$. Then, assuming the conditions of Theorem \ref{thm:clt_general}, for either
$$\C_\alpha^+ = \left\{\theta:   \left\|\hat V_{\theta}^{-1/2} \nabla L^+(\theta) \right\|^2 \leq \frac{\chi^2_{d,1-\alpha}}{n} \right\} \quad \text{ or } \quad \C_\alpha^+ = \left(\hat\theta_i^+ \pm z_{1-\alpha/(2d)} \sqrt{\frac{\hat\Sigma_{ii}}{n}},\right)_{i=1}^d,$$
we have
$$\liminf_{n,N} ~\P\left(\theta^* \in \C_\alpha^+\right) \geq 1-\alpha.$$
\end{cor}

Above, $\chi^2_{d,1-\alpha}$ is the $(1- \alpha)$-quantile of the chi-squared distribution with $d$ degrees of freedom; when $d=1$ (as in the case of mean estimation), $\chi_{d,1-\alpha}$ is equal to $z_{1-\alpha/2}$. Note also that in the case of mean estimation, the two confidence sets are identical and reduce to the set from Corollary \ref{cor:mean_inf}. In the second confidence set we apply a Bonferroni correction over the $d$ coordinates of the estimand for simplicity and clarity of exposition; we can obtain an asymptotically exact $(1-\alpha)$-confidence set as $\C_\alpha^+ = \left\{\hat\theta^+ + v :  v^\top \hat\Sigma v \leq \frac{\chi^2_{d,1-\alpha}}{n}  \right\}$.

Next, we apply Theorem \ref{thm:clt_general} and Corollary \ref{cor:confint_general} to concrete problems---quantile estimation, linear regression, and generalized linear models---to get explicit confidence interval constructions.

\subsection{Example: quantile estimation}
\label{sec:quantile}

Assume we are interested in a quantile of $Y$,  
$$\theta^* = \inf\left\{y: \P(Y\leq y)\geq q \right\}.$$
The quantile can equivalently be written as any minimizer of the pinball loss, 
$$\theta^* = \argmin_\theta \E[\ell_\theta(Y)] = \argmin_\theta \E[q(Y-\theta)\mathbf{1}\{Y>\theta\} + (1-q)(\theta-Y)\mathbf{1}\{Y\leq\theta\}].$$
The subgradient of the pinball loss is equal to $\nabla \ell_\theta(y) = - q\mathbf{1}\{y>\theta\} + (1-q)\mathbf{1}\{y\leq \theta\} = -q + \mathbf{1}\{y\leq \theta\}$. Plugging this expression into the first confidence set from Corollary \ref{cor:confint_general} yields
$$\C_\alpha^+ = \left\{ \theta: \left|\tilde F^+(\theta) - \Delta^+(\theta)  - q \right| \leq z_{1- \alpha/2} \frac{\sqrt{\frac n N \hat\sigma_\theta^2 + \hat \sigma_{\Delta, \theta}^2 }}{\sqrt n} \right\},$$
where $\tilde F^+(\theta) = \frac{1}{K N} \sum_{j=1}^K \sum_{i=1}^N \mathbf{1}\{f^{(j)}(\Xt_i)\leq\theta\}$ is the average empirical CDF of the predictions on the unlabeled data, and $\Delta^+(\theta) = \frac 1 n \sum_{j=1}^K \sum_{i\in I_j} (\mathbf{1}\{f^{(j)}(X_i) \leq \theta\} - \mathbf{1}\{Y_i \leq \theta\})$ is the difference between the empirical CDFs of the predictions and true outcomes on the labeled data. The standard errors are equal to $\bar\sigma^2_\theta = \Var(\mathbf{1}\{\bar f(X) \leq \theta\})$ and $\bar\sigma^2_{\Delta,\theta} = \Var(\mathbf{1}\{ \bar f(X) \leq \theta \} - \mathbf{1}\{ Y \leq \theta \})$. The confidence set $\C_\alpha^+$ thus consists of all values $\theta$ such that the average predicted CDF $\tilde F^+(\theta)$, corrected by the bias $\Delta^+(\theta)$, is close to the target level  $q$.

\subsection{Example: linear regression}

In linear regression, the target of inference is defined by
\begin{equation}
\label{eq:linear_reg}
\theta^* = \argmin_\theta \E[\ell_\theta(X,Y)] = \argmin_\theta \frac 1 2 \E[(Y-X^\top \theta)^2].
\end{equation}
In this case, the cross-prediction estimator, equal to the solution to $\nabla L^+(\hat\theta^+)=0$, has a closed-form expression. Letting $\tilde \Xbb\in \R^{N\times d}$ (resp.~$\Xbb\in\R^{n\times d}$) be the unlabeled (resp.~labeled) data matrix, $\Ybb \in\R^n$ be the vector of labeled outcomes, the solution is given by
\begin{equation*}
\hat\theta^+ = (\tilde{\Xbb}^\top\tilde{\Xbb})^{-1}\left(\tilde \Xbb^\top  f_{\mathrm{avg(K)}}(\tilde \Xbb)  - \frac N n \cdot \Xbb^\top \left(f_{1:K}(\Xbb) - \Ybb \right) \right),
\end{equation*}
where $f_{\mathrm{avg(K)}}(\tilde \Xbb) = \frac 1 K \sum_{j=1}^K f^{(j)}(\tilde \Xbb)$ is the vector of average predictions on the unlabeled data, and $f_{1:K}(\Xbb)$ is the vector of predictions on the labeled data: $f_{1:K}(\Xbb) = (f^{(1)}(X_1),\dots,f^{(1)}(X_{\frac n K}),f^{(2)}(X_{\frac n K +1}),\dots,f^{(K)}(X_n))$. We see that $\hat\theta^+$ resembles the usual least-squares estimator with $f_{\mathrm{avg(K)}}(\tilde \Xbb)$ as the response, except for the extra debiasing factor, $\frac N n \cdot \Xbb^\top \left(f_{1:K}(\Xbb) - \Ybb \right)$, that takes into account the prediction inaccuracies.

Instantiating the relevant terms, Theorem \ref{thm:clt_general} implies that $\hat\theta^+$ is asymptotically normal with covariance equal to $\Sigma_{\mathrm{OLS}} = H^{-1} \bar V_{\theta^*} H^{-1}$, where $H = \E[XX^\top]$ and $\bar V_{\theta^*} = \frac{n}{N} \bar\Sigma_{\theta^*} + \bar \Sigma_\Delta$, for
$\bar \Sigma_{\theta^*} = \Var((\bar f(X) - X^\top \theta^*)X)$ and $\bar \Sigma_{\Delta} = \Var((\bar f(X) - Y)X)$.

For a given coordinate of interest $i$, a confidence interval for $\theta^*_i$ can therefore be obtained as
$$\C_\alpha^+ = \left(\hat\theta^+_i \pm z_{1-\alpha/2} \frac{\sqrt{(\Sigmaolshat)_{ii}}}{\sqrt{n}}\right),$$
given an estimate $\Sigmaolshat$ of $\Sigmaols$.

\subsection{Example: generalized linear models}

We can generalize the previous example by considering all generalized linear models (GLMs). In particular, we consider targets of inference given by
\begin{equation}
\label{eq:glms}
\theta^* = \argmin_\theta \E[-\log p_{\theta}(Y | X)] = \argmin_\theta \E[-Y\theta^\top X + \psi(X^\top \theta)],
\end{equation}
where $p_{\theta}(y|x) = \exp(yx^\top \theta-\psi(x^\top\theta))$ is the probability density of the outcome given the features and the log-partition function $\psi$ is convex. The objective \eqref{eq:glms} recovers the linear-regression objective \eqref{eq:linear_reg} by setting $\psi(s) = \frac 1 2 s^2$. It captures logistic regression by choosing $\psi(s)=\log(1+e^s)$.

The asymptotic covariance given by Theorem \ref{thm:clt_general} evaluates to $\Sigma_{\mathrm{GLM}} =  H_{\theta^*}^{-1} \bar V_{\theta^*}  H_{\theta^*}^{-1}$, $H_{\theta^*} = \E[\psi''(X^\top \theta^*)X X^\top]$, $\bar V_{\theta^*} = \frac n N \Var((\psi'(X^\top \theta^*) - \bar f(X)) X) + \Var((\bar f(X) - Y)X)$. Therefore, analogously to the OLS case, given an estimate $\hat\Sigma_{\mathrm{GLM}}$ of $\Sigma_{\mathrm{GLM}}$, we can construct a confidence interval for $\hat\theta^+$ as
$$\C_\alpha^+ = \left(\hat\theta^+_i \pm z_{1-\alpha/2} \frac{\sqrt{(\hat\Sigma_{\mathrm{GLM}})_{ii}}}{\sqrt{n}}\right).$$

\section{Variance estimation via bootstrapping}
\label{sec:variance_est}

The previous inference results rely on being able to estimate the asymptotic covariance of $\hat \theta^+$. We herewith provide an explicit estimation strategy that we will use in our experiments.

Recall that that the asymptotic covariance is equal to $\bar \Sigma = H_{\theta^*}^{-1}\bar V_{\theta^*} H_{\theta^*}^{-1}$, where
$\bar V_{\theta} = \frac n N \bar\Sigma_\theta + \bar \Sigma_{\Delta,\theta}$, for $\bar \Sigma_{\theta} = \Var(\nabla \ell_{\theta,i}^{\bar f})$ and $\bar \Sigma_{\Delta,\theta} = \Var(\nabla \ell_{\theta,i}^{\bar f} - \nabla \ell_{\theta,i})$.
Estimating the Hessian $H_{\theta}$ is easy via plug-in estimation; $\bar V_{\theta}$, on the other hand, depends on the average model $\bar f$. If the average model $\bar f$ was known, one could compute estimates of $\bar \Sigma_{\theta}$ and $\bar \Sigma_{\Delta,\theta}$ by replacing the true covariances with their empirical counterparts. Thus, the challenge is to approximate $\bar f$. To achieve this, we apply the bootstrap to simulate multiple model training runs, and at the end we average the predictions of all the learned models.

In more detail, for each $b \in \{1, 2, \ldots, B\} = [B]$, we sample $n-\frac n K$ data points uniformly at random from the labeled dataset, and denote the indices of the samples by $I_{\mathrm{boot}}^b$. Then, we use the sampled data points to train a model $f_{\mathrm{boot}}^{(b)}$ using the same model-fitting strategy as for the cross-prediction models $f^{(j)}$. To estimate $\bar \Sigma_{\theta}$, we compute
$$\hat \Sigma_{\theta} = \widehat\Var\left(\nabla \ell_{\theta}(\tilde X_i,\bar f_{\mathrm{boot}}(\tilde X_i)), i\in [N]\right),$$
where $\bar f_{\mathrm{boot}} =
\frac{1}{B} \sum_{b=1}^B f^{(b)}_{\mathrm{boot}}$ and $\widehat\Var$ denotes the empirical covariance. To estimate $\bar\Sigma_{\Delta,\theta}$, we compute
$$\hat \Sigma_{\Delta,\theta} = \widehat\Var\left(\nabla \ell_{\theta}( X_i, f_{\mathrm{boot}}^{(b)}(X_i)) - \nabla \ell_{\theta}( X_i, Y_i), i\in [n]\setminus I_{\mathrm{boot}}^{(b)}, b\in [B]\right).$$
Finally, we approximate the covariance by $\frac n N \hat\Sigma_{\theta} + \hat \Sigma_{\Delta,\theta}$. In computing $\hat \Sigma_{\Delta,\theta}$, we technically do not average out the bootstrapped models because we want to make sure that every point $(X_i,Y_i)$ used to compute the gradient bias is independent of its corresponding model $f^{(b)}_{\mathrm{boot}}$. Intuitively, as per the discussion following Assumption \ref{ass:stability}, if $\A_{\mathrm{train}}$ is stable we expect $f^{(b)}_{\mathrm{boot}}$ to be a good approximation of the average model $\bar f$, which in turn means that the bootstrap covariance estimates should be consistent per conventional wisdom. We show empirically that the covariance estimates lead to valid coverage across a range of applications.

To give one concrete example, consider mean estimation: $\theta^* = \E[Y]$. We compute 
$$\hat \sigma^2 = \widehat\Var\left(\bar f_{\mathrm{boot}}(\tilde X_i), i\in [N]\right)\text{ and }\hat \sigma_{\Delta}^2 = \widehat\Var\left(f_{\mathrm{boot}}^{(b)}(X_i) - Y_i, i\in [n]\setminus  I_{\mathrm{boot}}^{(b)}, b\in [B]\right),$$
and form the final interval as $\C_\alpha^+ = \left(\hat\theta^+ \pm z_{1-\alpha/2} \frac{\sqrt{\frac n N \hat \sigma^2 + \hat \sigma_{\Delta}^2}}{\sqrt{n}}\right)$.

\section{Experiments}
\label{sec:experiments}

We evaluate cross-prediction and compare it to baseline approaches on several datasets; the baselines are the classical inference method, which only uses the labeled data, and prediction-powered inference with a data-splitting step in order to train a predictive model. Code for reproducing the experiments is available at:  \href{https://github.com/tijana-zrnic/cross-ppi}{\texttt{https://github.com/tijana-zrnic/cross-ppi}}.

For each experimental setting, we plot the coverage and confidence interval width estimated over $100$ trials for all baselines. We also show the constructed confidence intervals for five randomly chosen trials. Finally, to quantify the stability of inferences, we report the standard deviation of the lower and upper endpoints of the confidence intervals for each method.

We begin with proof-of-concept experiments on synthetic data. Then, we move on to more complex real datasets.

\subsection{Proof-of-concept experiments on synthetic data}

To build intuition, we begin with simple experiments on synthetic data. The purpose is to confirm what we expect in theory: (a) as it gets easier to predict labels from features, cross-prediction and prediction-powered inference become more powerful and increasingly outperform the classical approach; (b) cross-prediction uses the data more efficiently than prediction-powered inference, yielding smaller intervals; (c) cross-prediction gives more stable inferences than the baselines when the predictions are useful; (d) all three approaches lead to satisfactory coverage.

In all of the following experiments, we have $N=10,000$ unlabeled data points, and we vary the size of the labeled data $n$ between $100$ and $1000$, in $100$-point increments. We apply cross-prediction with $K=10$ folds. We estimate the variance using the bootstrap approach described in Section \ref{sec:variance_est}, with $B=30$ bootstrap samples. For prediction-powered inference, we use half of the labeled data for model training. To illustrate the point that cross-prediction can be applied with any black-box model, we train gradient-boosted trees via XGBoost \cite{chen2016xgboost} to obtain the models $f^{(j)}$. We use the same model-fitting strategy for prediction-powered inference. We fix the target error level to be $\alpha=0.1$ and average the results over $100$ trials.

\paragraph{Mean estimation.} For given parameters $R^2$ and $\sigma_Y^2$, the data-generating distribution is defined as $X\sim \mathcal N(0,I_2), Y = \mu + X^\top \beta + \xi$, where $\beta_ 1  = \beta_2 =  R\sigma_Y/\sqrt{2}$, and $\xi\sim \mathcal N(0,\sigma_Y^2(1 - R^2))$ is independent of $X$. We fix $\mu = 4, \sigma_Y^2 = 4$ and vary $R^2 = \frac{\Var(X^\top \beta)}{\Var(Y)} \in\{0,0.5,1\}$. The idea is to vary the degree to which the outcomes can be explained through the features: when $R^2 = 0$, the outcome is independent of the features and we do not expect predictions to help, while when $R^2 = 1$, the outcome can be perfectly explained through the features and we expect predictions to be helpful. Given that the variance of $Y$ is kept constant regardless of $R^2$, classical inference has the same distribution of widths across $R^2$. The target of inference is $\theta^* = \E[Y] = \mu$.

In Figure \ref{fig:mean_comparison} we plot the coverage and interval width of cross-prediction, classical inference, and prediction-powered inference, as well as five example intervals. All three methods approximately achieve the target coverage, and cross-prediction dominates prediction-powered inference throughout. Further, we see that the classical approach dominates the alternatives when the features are independent of the outcomes, while the alternatives become more powerful as $R^2$ increases. To evaluate stability, in Table \ref{table:mean} we report the standard deviation of the lower and upper endpoints of the confidence intervals from Figure \ref{fig:mean_comparison}, for $n=100$. We observe that the classical approach is the most stable method when $R^2=0$, which makes sense because the predictions can only introduce noise. When $R^2=0.5$, cross-prediction and classical inference have a similar degree of stability, while when $R^2=1$ cross-prediction is significantly more stable. Moreover, cross-prediction is significantly more stable than prediction-powered inference throughout. These trends hold across different values of $n$, however we only include the results for $n=100$ for simplicity of exposition.

\begin{figure}[t]
\centering
\includegraphics[width = \textwidth]{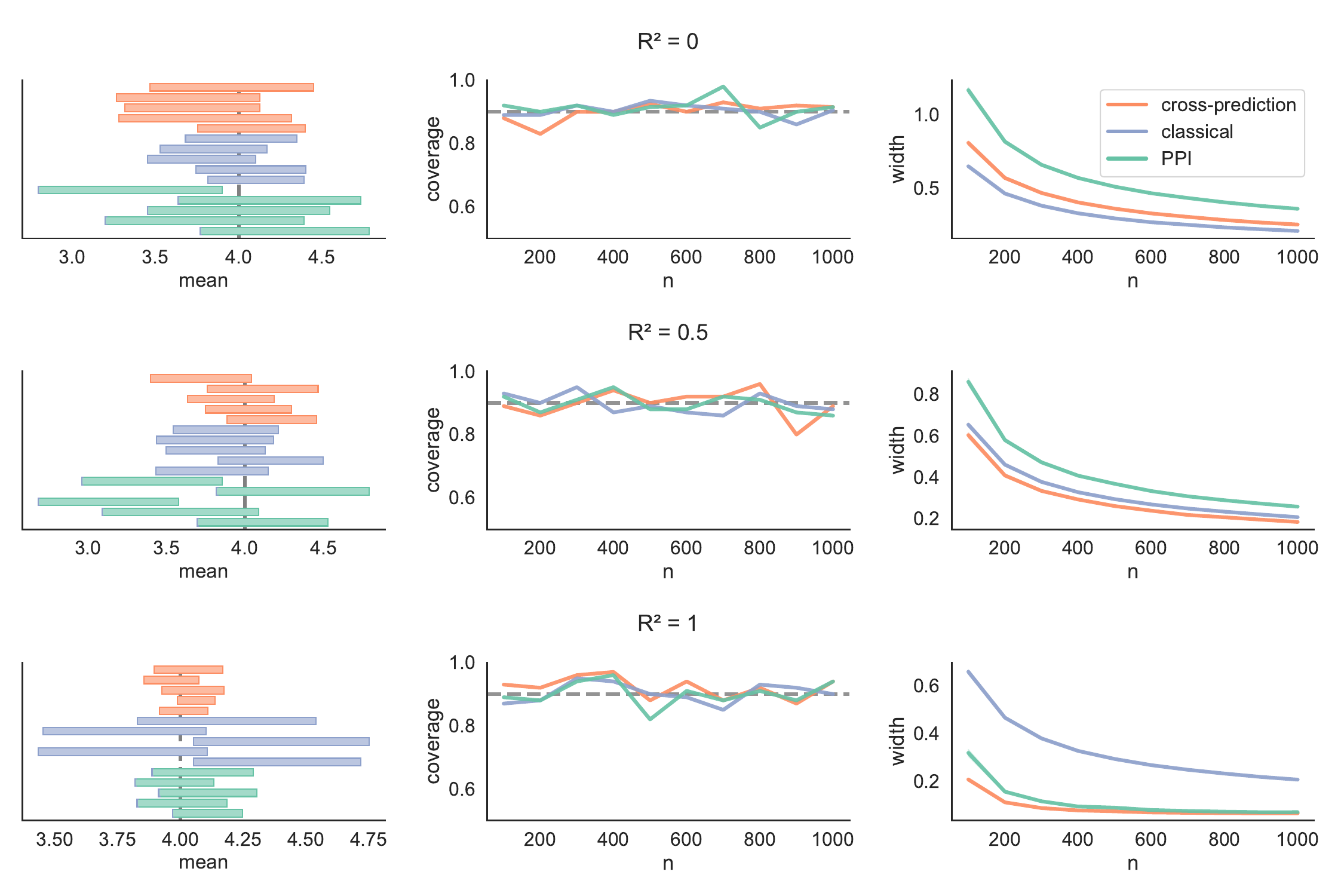}
\caption{\textbf{Mean estimation.} Intervals from five randomly chosen trials (left), coverage (middle), and average interval width (right) of cross-prediction, classical inference, and prediction-powered inference (PPI) in a mean estimation problem. 
}
\label{fig:mean_comparison}
\end{figure}

{\renewcommand{\arraystretch}{1.2}\begin{table}[b!]
\centering
\begin{tabularx}{\textwidth}{ |c| Y Y| Y Y | Y Y| }
\cline{1-7}
\hline
\rowcolor{Gray}
        & \multicolumn{6}{c|}{\textbf{Mean estimation}}\\
        \hhline{|>{\arrayrulecolor{Gray}}->{\arrayrulecolor{black}}------|}
\rowcolor{Gray}
\textbf{}        & \multicolumn{2}{c|}{$\mathbf{R^2 = 0}$}          & \multicolumn{2}{c|}{$\mathbf{R^2 = 0.5}$}             & \multicolumn{2}{c|}{$\mathbf{R^2 = 1}$}                 \\ \hline
\rowcolor{LightGray}
\textbf{Method}  & \multicolumn{1}{c|}{$\hat\sigma_L$} & $\hat\sigma_U$ & \multicolumn{1}{c|}{$\hat\sigma_L$} & $\hat\sigma_U$ & \multicolumn{1}{c|}{$\hat\sigma_L$} & $\hat\sigma_U$ \\ \hline
cross-prediction &      0.2694 & 
0.2696 & 
\textbf{0.1769} & 
0.1897 & 
\textbf{0.0591} & 
\textbf{0.0613}
              \\ 
classical        &    \textbf{0.2124} & 
\textbf{0.2085} & 
0.1908 & 
\textbf{0.1885} & 
0.2136 & 
0.2102
         \\ 
PPI              &   0.3844 & 
0.3997 & 
0.2751 & 0.2684 & 
0.1045 & 
0.1061           \\ \hline
\end{tabularx}
\caption{Standard deviation of the lower ($\hat\sigma_L$) and upper ($\hat\sigma_U$) endpoints of the confidence intervals in the mean estimation problem from Figure \ref{fig:mean_comparison}, for $n=100$. The minimum value in each column is in bold.}
\label{table:mean}
\end{table}}

\begin{figure}[t]
\centering
\includegraphics[width = \textwidth]{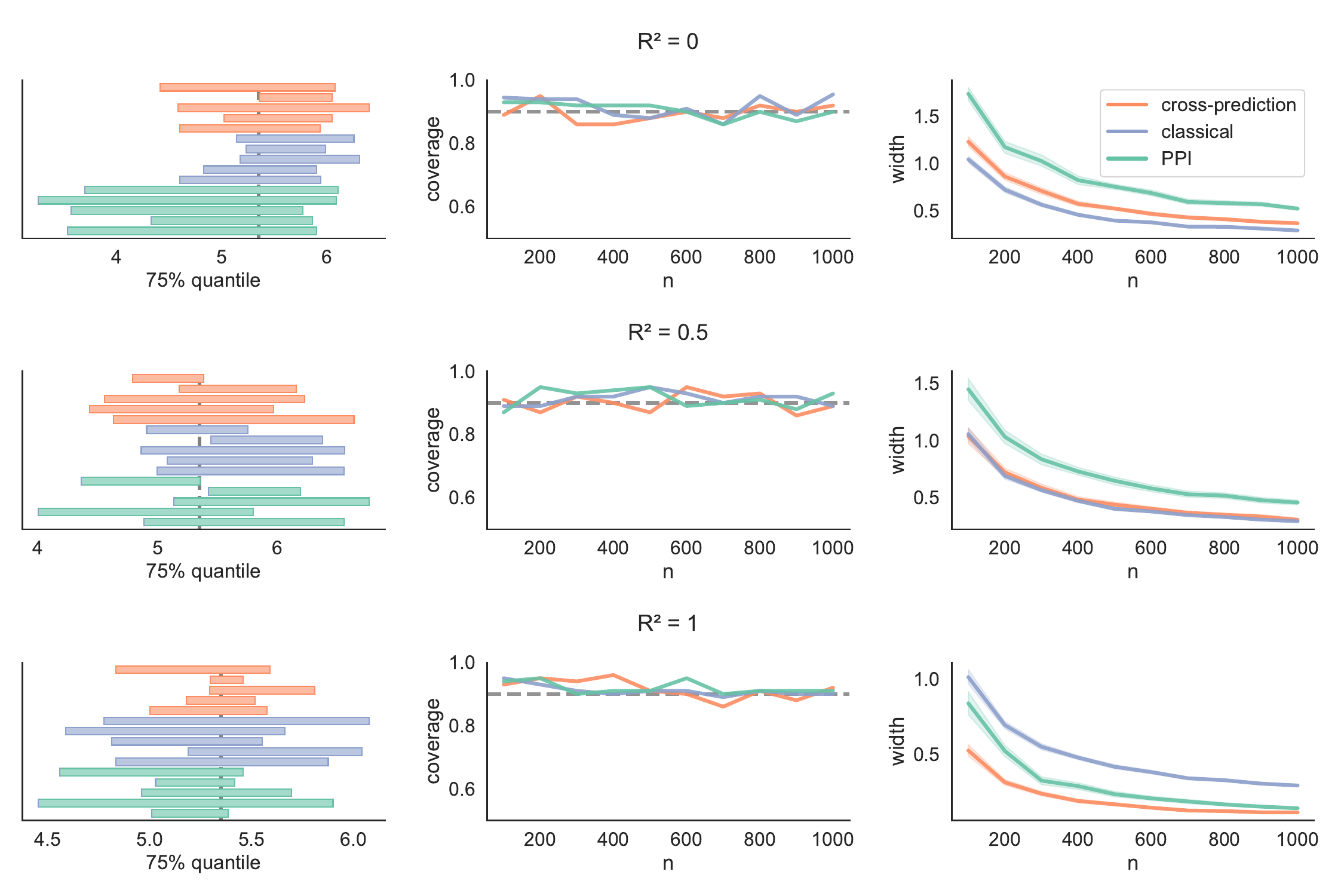}
\caption{\textbf{Quantile estimation.} Intervals from five randomly chosen trials (left), coverage (middle), and average interval width (right) of cross-prediction, classical inference, and prediction-powered inference (PPI) in a quantile estimation problem. The target is the 75th percentile.}
\label{fig:quantile_comparison}
\end{figure}

\begin{figure}[t]
\centering
\includegraphics[width = \textwidth]{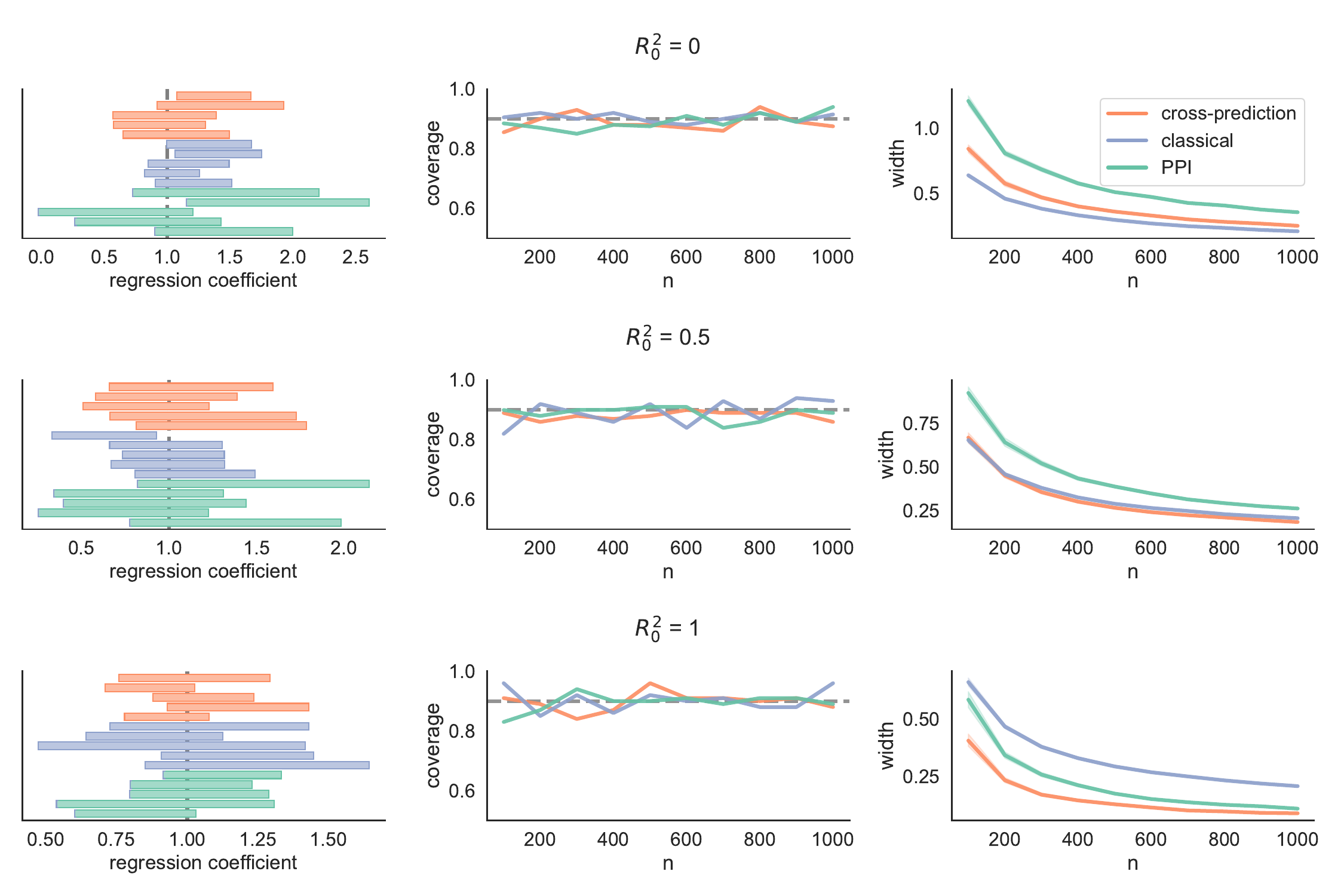}
\caption{\textbf{Linear regression.} Intervals from five randomly chosen trials (left), coverage (middle), and average interval width (right) of cross-prediction, classical inference, and prediction-powered inference (PPI) in a linear regression problem.}
\label{fig:linear_regression}
\end{figure}

{\renewcommand{\arraystretch}{1.2}\begin{table}[b!]
\centering
\begin{tabularx}{\textwidth}{ |c| Y Y| Y Y | Y Y| }
\cline{1-7}
\hline
\rowcolor{Gray}
        & \multicolumn{6}{c|}{\textbf{Quantile estimation}}\\
\hhline{|>{\arrayrulecolor{Gray}}->{\arrayrulecolor{black}}------|}
\rowcolor{Gray}
\textbf{}        & \multicolumn{2}{c|}{$\mathbf{R^2 = 0}$}          & \multicolumn{2}{c|}{$\mathbf{R^2 = 0.5}$}             & \multicolumn{2}{c|}{$\mathbf{R^2 = 1}$}                 \\ \hline
\rowcolor{LightGray}
\textbf{Method}  & \multicolumn{1}{c|}{$\hat\sigma_L$} & $\hat\sigma_U$ & \multicolumn{1}{c|}{$\hat\sigma_L$} & $\hat\sigma_U$ & \multicolumn{1}{c|}{$\hat\sigma_L$} & $\hat\sigma_U$ \\ \hline
cross-prediction &      0.4102 &
0.3509 &
0.3253 &
\textbf{0.3242} & 
\textbf{0.1345} &
\textbf{0.1545}
              \\ 
classical        &    \textbf{0.2302} &
\textbf{0.3024} &
\textbf{0.2569} &
0.3305 &
0.2615 &
0.2806
         \\ 
PPI              &   0.5424 &
0.4614 &
0.4141 &
0.4368 & 
0.2151 &
0.3280           \\ \hline
\end{tabularx}
\caption{Standard deviation of the lower ($\hat\sigma_L$) and upper ($\hat\sigma_U$) endpoints of the confidence intervals in the quantile estimation problem from Figure \ref{fig:quantile_comparison}, for $n=100$. The minimum value in each column is in bold.}
\label{table:quantile}
\end{table}}

{\renewcommand{\arraystretch}{1.2}\begin{table}[b!]
\centering
\begin{tabularx}{\textwidth}{ |c| Y Y| Y Y | Y Y| }
\cline{1-7}
\hline
\rowcolor{Gray}
        & \multicolumn{6}{c|}{\textbf{Linear regression}}\\
        \hhline{|>{\arrayrulecolor{Gray}}->{\arrayrulecolor{black}}------|}
\rowcolor{Gray}
\textbf{}        & \multicolumn{2}{c|}{$\mathbf{R_0^2 = 0}$}          & \multicolumn{2}{c|}{$\mathbf{R_0^2 = 0.5}$}             & \multicolumn{2}{c|}{$\mathbf{R_0^2 = 1}$}                 \\ \hline
\rowcolor{LightGray}
\textbf{Method}  & \multicolumn{1}{c|}{$\hat\sigma_L$} & $\hat\sigma_U$ & \multicolumn{1}{c|}{$\hat\sigma_L$} & $\hat\sigma_U$ & \multicolumn{1}{c|}{$\hat\sigma_L$} & $\hat\sigma_U$ \\ \hline
cross-prediction &      0.2801 &
0.2969 &
\textbf{0.1875} &
\textbf{0.2250} &
\textbf{0.1102} &
\textbf{0.1472}
              \\ 
classical        &    \textbf{0.2091} &
\textbf{0.2098} &
0.2271 &
0.2262 &
0.1800 &
0.1809
         \\ 
PPI              &   0.4104 &
0.4870 &
0.2602 &
0.3326 &
0.1522 &
0.2530           \\ \hline
\end{tabularx}
\caption{Standard deviation of the lower ($\hat\sigma_L$) and upper ($\hat\sigma_U$) endpoints of the confidence intervals in the linear regression problem from Figure \ref{fig:linear_regression}, for $n=100$. The minimum value in each column is in bold.}
\label{table:linear_regression}
\end{table}}

\paragraph{Quantile estimation.} We adopt the same data-generating process as for mean estimation. We only change the target of inference $\theta^*$ to be the 75th percentile of the outcome distribution.

In Figure \ref{fig:quantile_comparison} we plot the coverage and interval width of cross-prediction, classical inference, and prediction-powered inference, as well as five example intervals. We observe a qualitatively similar comparison as in the case of mean estimation: all three methods approximately achieve the target coverage, and cross-prediction dominates prediction-powered inference throughout. As before, the classical approach dominates the alternatives when the features are independent of the outcomes, and the alternatives become increasingly powerful as $R^2$ increases. In Table \ref{table:quantile} we evaluate the stability of the methods by reporting the standard deviation of the lower and upper endpoints of the confidence intervals from Figure \ref{fig:quantile_comparison}, for $n=100$. As before, Table \ref{table:quantile} shows that cross-prediction is more stable than prediction-powered inference for all values of $R^2$, and when $R^2=0$ classical inference is the most stable option. When $R^2=0.5$, cross-prediction has a slightly more stable upper endpoint than classical inference, while classical inference has a more stable lower endpoint. When $R^2=1$, cross-prediction is by far the most stable method. For $R^2\in\{0,0.5\}$. Again, these trends are largely consistent across different values of $n$, however we only include the results for $n=100$ for simplicity.

\begin{figure}[bh!]
\centering
\includegraphics[width = 0.6\textwidth]{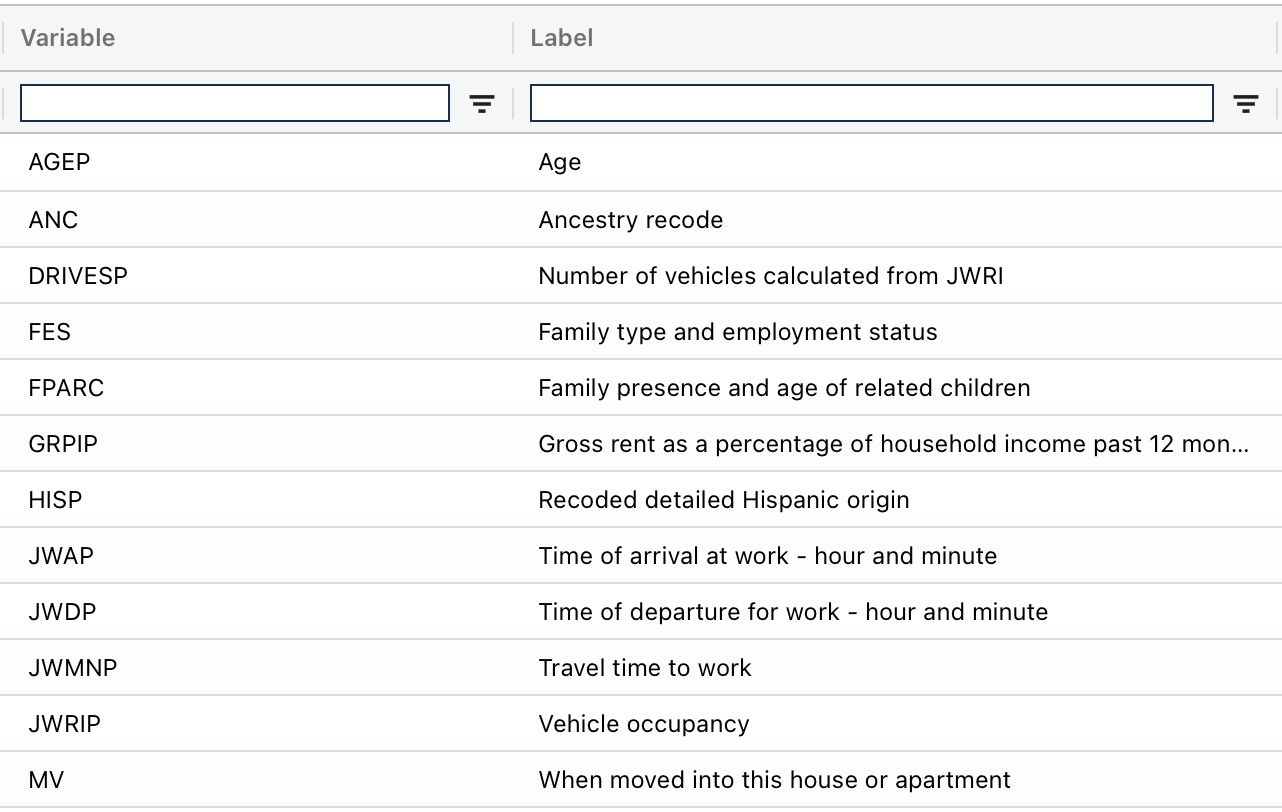}
\caption{Subset of the variables available in the ACS PUMS data.}
\label{fig:census_variables}
\end{figure}

\paragraph{Linear regression.} Finally, we look at linear regression. For robustness and interpretability, it is common to include only a subset of the available features in the regression. The process of deciding which variables to include is known as model selection. The variables that are not included in the model may still be predictive of the outcome of interest; we demonstrate that, as such, they can be useful for inference.

The data-generating distribution is defined as follows: we generate $X\sim\mathcal N(0,I_3)$, $Y= X^\top \beta + \xi$, where $\beta = (1,1,R_0\sigma_Y)$ and $\xi\sim\mathcal N(0,\sigma_Y^2(1 - R_0^2))$. Again, the idea is to vary how much of the outcome can be explained through prediction versus how much of it is exogenous randomness. We fix $\sigma_Y^2 = 4$ and vary $R_0\in\{0,0.5,1\}$. The target of inference is defined as the least-squares solution when regressing $Y$ on $(X_1,X_2)$, that is, the first coordinate of this solution. In this case, this is simply equal to $\theta^* = \beta_1 = 1$.

In Figure \ref{fig:linear_regression} we plot the coverage and interval width of cross-prediction, classical inference, and prediction-powered inference. When $R_0^2=0$, the classical approach outperforms the prediction-based approaches; as $R_0^2$ grows, meaning more of the randomness of the outcome can be attributed to $X_3$, the prediction-based approaches dominate the classical one. As before, cross-prediction yields smaller intervals than prediction-powered inference. We remark that, even though the inference problem posits a linear model, the prediction-based approaches still use XGBoost for model training. Like in the previous two examples, we report on the stability of the three methods in Table \ref{table:linear_regression}. We again fix $n=100$ for simplicity. Cross-prediction is far more stable than prediction-powered inference throughout, and it is more stable than classical inference for nonzero values of $R_0^2$.

\subsection{Estimating deforestation from satellite imagery}
\label{sec:deforestation}

We briefly revisit the problem of deforestation analysis from Section \ref{sec:intro}. As we saw in Figure \ref{fig:deforestation}, cross-prediction gave tighter confidence intervals for the deforestation rate than using gold-standard measurements of deforestation alone. In other words, cross-prediction can enable a reduction in the number of necessary field visits to measure deforestation. Moreover, we saw that cross-prediction outperformed prediction-powered inference.

Here we argue another benefit of cross-prediction in this problem: it is a more stable solution than the baselines.
Table \ref{table:realdata} shows the standard deviation of the endpoints of the confidence intervals constructed by cross-prediction and its competitors. Cross-prediction has a significantly lower variability of the endpoints than both classical inference and prediction-powered inference, while the latter two exhibit similar variability.

Finally, we provide the experimental details that were omitted in Section \ref{sec:intro} for brevity. We have $n_{\mathrm{all}} = 3192$ data points with gold-standard labels total. To simulate having unlabeled images, in each trial we randomly split the data into $n$ points to serve as the labeled data, for varying $n \in \{0.1n_{\mathrm{all}}, 0.2n_{\mathrm{all}}, 0.3n_{\mathrm{all}}\}$, and treat the remaining points as unlabeled. The target of inference is the fraction of deforested areas across the locations contained in the sample.
We apply cross-prediction with $K=10$ folds. For prediction-powered inference, we use $\ntr = 0.1n$ points for model tuning. We average the results over $100$ trials.

\begin{figure}[t]
\centering
\includegraphics[width = \textwidth]{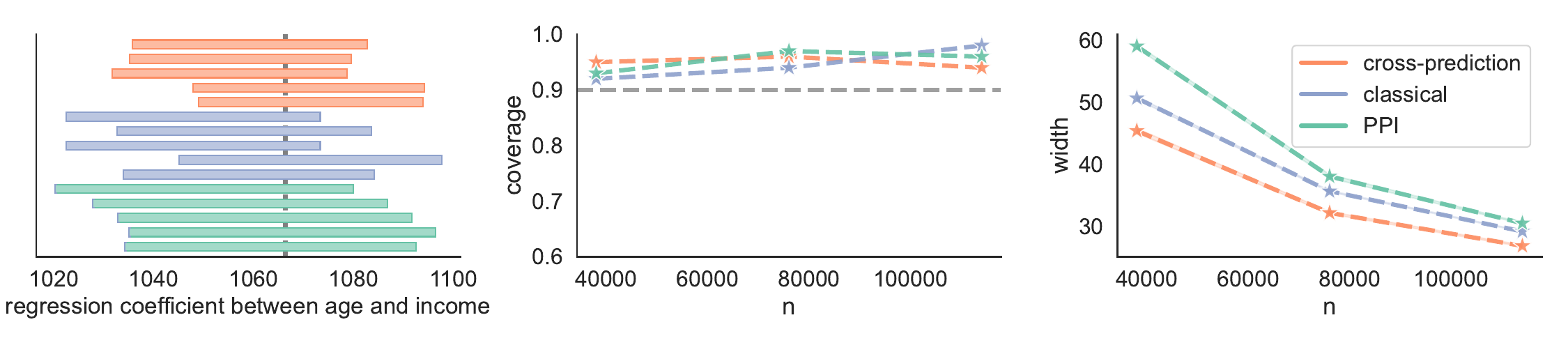}
\caption{\textbf{Estimating the relationship between age, sex, and income in ACS survey data.} Intervals from five randomly chosen trials (left), coverage (middle), and average interval width (right) of cross-prediction, classical inference, and prediction-powered inference (PPI) in a linear regression problem on ACS PUMS data. The target $\theta^*$ is the linear regression coefficient when regressing income on age and~sex.}
\label{fig:census_linear_regression}
\end{figure}

\subsection{Estimating relationships between socioeconomic covariates in survey data}

We evaluate cross-prediction on the American Community Survey (ACS) Public Use Microdata Sample (PUMS). We investigate the relationship between age, sex, and income in survey data collected in California in 2019 ($n_{\mathrm{all}} = 377,575$ people total). High-quality survey data is generally difficult and time-consuming to collect.
With this experiment we hope to demonstrate how, by imputing missing information based on the available covariates, cross-prediction can achieve both powerful and valid inferences while reducing the requisite amount of survey data. See Figure \ref{fig:census_variables} for a subset of the available covariates in the ACS PUMS data.

We use the Folktables \cite{ding2021retiring} interface to download the PUMS data, including income, age, sex, and 15 other demographic covariates. In each trial, we randomly sample $n$ data points to serve as the labeled data, for varying $n$, and treat the remaining data points as the unlabeled data. We vary $n\in\{0.1n_{\mathrm{all}}, 0.2n_{\mathrm{all}}, 0.3n_{\mathrm{all}}\}$. The target of inference is the linear regression coefficient when regressing income on age and sex: $\theta^* = \argmin_\theta \E[(Y - X_{\mathrm{ols}}^\top \theta)^2]$, where $Y$ is income and $X_{\mathrm{ols}}$ encodes age and sex, $X_{\mathrm{ols}} = (X_{\mathrm{age}}, X_{\mathrm{sex}})$. For the purpose of evaluating coverage, the corresponding coefficient computed on the whole dataset is taken as the ground-truth value of the estimand.
To obtain the models $f^{(j)}$, we train gradient-boosted trees via XGBoost~\cite{chen2016xgboost}. Note that the predictors use all 17 covariates to predict the missing labels, even though the target of inference is only defined with respect to two covariates. We apply cross-prediction with $K=5$ folds. For prediction-powered inference, we use $\ntr = 0.2n$ points for model training, and we also train gradient-boosted trees. The target error level is $\alpha=0.1$ and we average the results over $100$ trials.

In Figure~\ref{fig:census_linear_regression} we plot the coverage and interval width for the three baselines, together with five example intervals. All three methods cover the true target with the desired probability. Moreover, as before, cross-prediction outperforms prediction-powered inference. In this example, the predictive power of the trained models is not high enough for prediction-powered inference to outperform the classical approach; cross-prediction, however, outperforms both. In Table \ref{table:realdata} we report on the stability of the three methods for $n=0.1n_{\mathrm{all}}$. We observe that cross-prediction is more stable than both alternatives. We also observe that prediction-powered inference has more stable intervals than the classical approach, despite the fact that they are wider on average.

\begin{figure}[b]
\centering
\includegraphics[width = 0.25\textwidth]{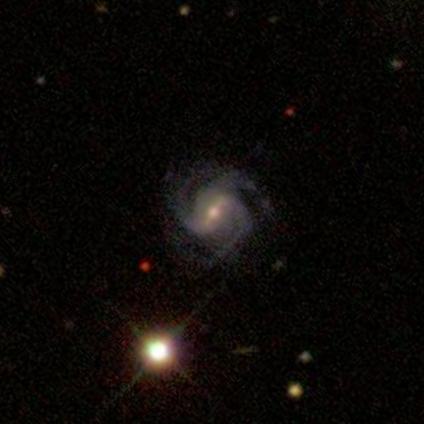}\hspace{20pt}
\includegraphics[width = 0.25\textwidth]{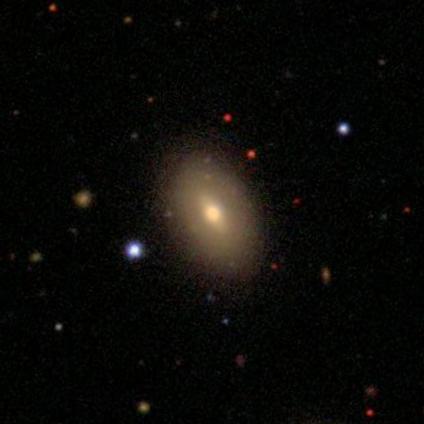}
\caption{Example images of a spiral galaxy (left) and a nonspiral galaxy (right).}
\label{fig:galaxy_examples}
\end{figure}

\subsection{Estimating the prevalence of spiral galaxies from galaxy images}

We next look at galaxy data from the Galaxy Zoo 2 dataset \cite{willett2013galaxy}, consisting of human-annotated images of galaxies from the Sloan Digital Sky 
Survey \cite{york2000sloan}.
Of particular interest are galaxies with spiral arms, which 
are correlated with star formation in the discs of low-redshift galaxies, and thus contribute to the 
understanding of star formation. See Figure \ref{fig:galaxy_examples} for example images of a spiral and a nonspiral galaxy. We show that, by leveraging the massive amounts of unlabeled galaxy imagery together with machine learning, cross-prediction can decrease the requisite number of human annotations for inference on galaxy demographics.

\begin{figure}[t]
\centering
\includegraphics[width = \textwidth]{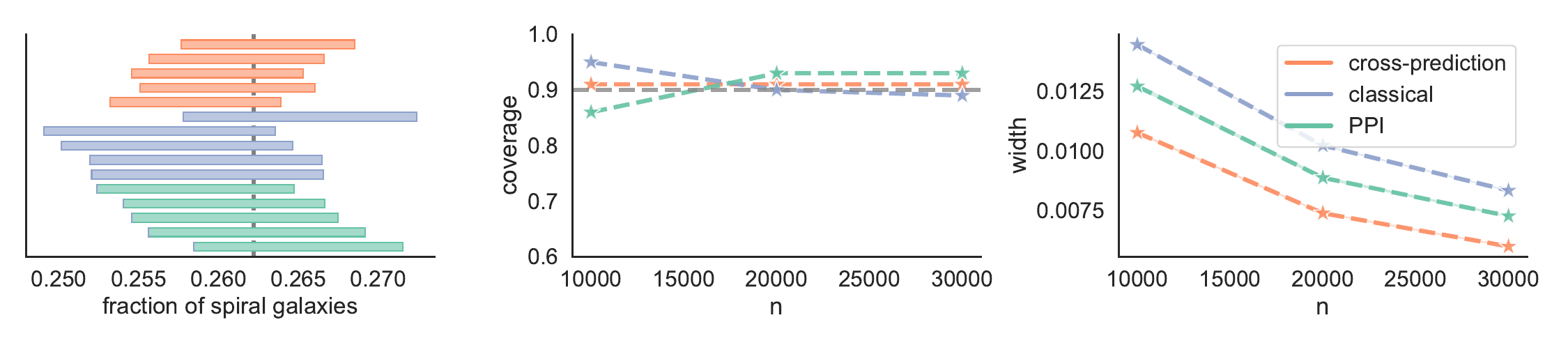}
\caption{\textbf{Estimating the prevalence of spiral galaxies from galaxy images.} Intervals from five randomly chosen trials (left), coverage (middle), and average interval width (right) of cross-prediction, classical inference, and prediction-powered inference (PPI) in a mean estimation problem on galaxy image data. The target $\theta^*$ is the fraction of spiral galaxies.}
\label{fig:galaxies}
\end{figure}

We have $167,434$ annotated galaxy images. In each trial, we randomly split them up into $n$ points to serve as the labeled data, for $n\in\{10000,20000,30000\}$, and treat the remaining data points as unlabeled. The target of inference is the fraction of spiral galaxies in the dataset, equal to about $26.22\%$. To compute predictions, we fine-tune all layers of a pre-trained ResNet50. We apply cross-prediction with $K=3$ folds. For prediction-powered inference, we use $\ntr = 0.1 n$ points for model training. The target error rate is $\alpha=0.1$ and we average the results over $100$ trials.

In Figure \ref{fig:galaxies} we plot the coverage and interval width of the three methods, as well as the intervals for five randomly chosen trials. Both cross-prediction and prediction-powered inference yield smaller intervals than the classical approach. Moreover, cross-prediction dominates prediction-powered inference.
We observe satisfactory coverage for all three procedures. In Table \ref{table:realdata} we evaluate the stability of the procedures for $n=10,000$. Cross-prediction is significantly more stable than classical inference and prediction-powered inference. The latter two achieve a similar degree of stability.

{\renewcommand{\arraystretch}{1.5}\begin{table}[t]
\centering
\begin{tabularx}{\textwidth}{ |c| Y Y| Y Y | Y Y| }
\cline{1-7}
\hline
\rowcolor{Gray}
\textbf{}        & \multicolumn{2}{c|}{\textbf{Deforestation Analysis}}          & \multicolumn{2}{c|}{\textbf{ACS Survey Analysis}}             & \multicolumn{2}{c|}{\textbf{Galaxy Analysis}}                 \\ \hline
\rowcolor{LightGray}
\textbf{Method}  & \multicolumn{1}{c|}{$\hat\sigma_L$} & $\hat\sigma_U$ & \multicolumn{1}{c|}{$\hat\sigma_L$} & $\hat\sigma_U$ & \multicolumn{1}{c|}{$\hat\sigma_L$} & $\hat\sigma_U$ \\ \hline
cross-prediction &      \textbf{0.0158}                                                     &     \textbf{0.0182}            &          \textbf{11.2781}
                                                   &    \textbf{12.2367}
            &          \textbf{0.0029}
                                                & \textbf{0.0029}
              \\ 
classical        &    0.0195                                                         &     0.0232           &     14.5346    &         15.6106       &        0.0037 &       0.0038
         \\ 
PPI              &   0.0200                                                          &    0.0240            &         13.1378                                                     &   13.8733              &         0.0036                            &       0.0037           \\ \hline
\end{tabularx}
\caption{Standard deviation of the lower ($\hat\sigma_L$) and upper ($\hat\sigma_U$) endpoints of the confidence intervals in the problems studied in Figure \ref{fig:deforestation}, Figure \ref{fig:census_linear_regression}, and Figure \ref{fig:galaxies}. For each problem we take $n$ to be the smallest labeled dataset size in the considered range. The minimum value in each column is in bold.}
\label{table:realdata}
\end{table}}

\section{Evaluating heuristics}

In Figure \ref{fig:deforestation}, we saw that cross-prediction gave tighter confidence intervals than the baseline approaches for the problem of deforestation analysis.
In this section, we test two heuristic ways of reducing the variance of the classical approach and prediction-powered inference and compare the heuristics to cross-prediction.

The first heuristic removes the debiasing from the cross-prediction estimator and simply averages the predictions on the large unlabeled dataset:
\begin{equation}
\label{eq:nodebias}
\hat\theta_{\mathrm{nodebias}} = \frac{1}{K  N} \sum_{j=1}^K \sum_{i=1}^N f^{(j)}(\Xt_i).
\end{equation}
This is akin to computing the classical estimator while pretending that the predicted labels are the ground truth.  
The second heuristic trains a model on all the labeled data, $f^{\mathrm{all}} = \mathcal{A}_{\mathrm{train}}(\{(X_i,Y_i)\}_{i=1}^n)$, and computes
\begin{equation*}
\hat\theta_{\mathrm{nofolds}} = \frac{1}{  N} \sum_{i=1}^N f^{\mathrm{all}}(\Xt_i) - \frac{1}{n} \sum_{i=1}^n (f^{\mathrm{all}}(X_i) - Y_i).
\end{equation*}
This estimator is akin to the prediction-powered estimator if we treated $f^{\mathrm{all}}$ as fixed and independent of the labeled dataset.

For both heuristics, we form confidence intervals based on the usual central limit theorem that assumes i.i.d. sampling. For the first heuristic this is done conditional on the trained models $f^{(j)}$, since the terms $(\frac 1 K \sum_{j=1}^K f^{(j)}(\tilde X_i))_{i\in[N]}$ are indeed conditionally independent given $f^{(1)},\dots,f^{(K)}$. Since the second heuristic proceeds under the assumption that $f^{\mathrm{all}}$ can be seen as being independent of the labeled data, we apply the central limit theorem to the two sums separately, as if $f^{\mathrm{all}}$ were fixed.

We see in Figure \ref{fig:heuristics} that removing the debiasing is detrimental to coverage; removing the folds has a more moderate effect that vanishes with $n$, but it is nevertheless significant.
Cross-prediction yields wider intervals than both heuristics, and by doing so it maintains correct coverage.

\begin{figure}[t]
\centering
\includegraphics[width = 0.75\textwidth]{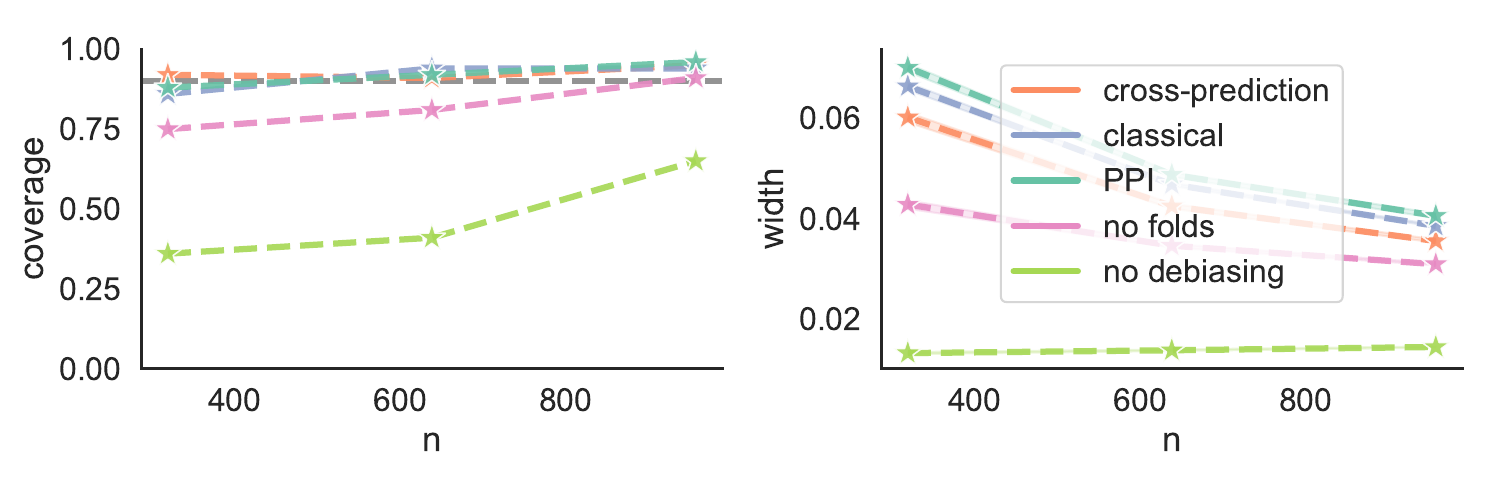}
\caption{\textbf{Estimating the deforestation rate in the Amazon from satellite imagery (revisited).} Coverage and average interval width of cross-prediction, classical inference, and prediction-powered inference (PPI), as well as two heuristics related to cross-fitting: one that removes the debiasing and one that trains on all labeled data instead of forming folds. The experimental setup is the same as in Figure~\ref{fig:deforestation}.}
\label{fig:heuristics}
\end{figure}

\subsection*{Acknowledgements}

We thank Anastasios Angelopoulos, Ying Jin, and Asher Spector for helpful suggestions and feedback on a draft on this manuscript, and Aditya Ghosh for catching and fixing a typo in a previous version of the manuscript. We also thank the reviewers for helping us clarify and improve the paper's exposition. T.Z.~was supported by Stanford Data Science through the Fellowship program. E.J.C.~was supported by the Office of Naval Research grant N00014-20-1-2157, the National Science Foundation grant DMS-2032014, the Simons Foundation under award 814641, and the ARO grant 2003514594.

\bibliographystyle{plainnat}
\bibliography{refs}

\begin{thebibliography}{60}
\providecommand{\natexlab}[1]{#1}
\providecommand{\url}[1]{\texttt{#1}}
\expandafter\ifx\csname urlstyle\endcsname\relax
  \providecommand{\doi}[1]{doi: #1}\else
  \providecommand{\doi}{doi: \begingroup \urlstyle{rm}\Url}\fi

\bibitem[Angelopoulos et~al.(2023{\natexlab{a}})Angelopoulos, Bates, Fannjiang,
  Jordan, and Zrnic]{angelopoulos2023prediction}
Anastasios~N Angelopoulos, Stephen Bates, Clara Fannjiang, Michael~I Jordan,
  and Tijana Zrnic.
\newblock Prediction-powered inference.
\newblock \emph{Science}, 382\penalty0 (6671):\penalty0 669--674,
  2023{\natexlab{a}}.

\bibitem[Angelopoulos et~al.(2023{\natexlab{b}})Angelopoulos, Duchi, and
  Zrnic]{angelopoulos2023ppipp}
Anastasios~N Angelopoulos, John~C Duchi, and Tijana Zrnic.
\newblock {PPI}++: Efficient prediction-powered inference.
\newblock \emph{arXiv preprint arXiv:2311.01453}, 2023{\natexlab{b}}.

\bibitem[Austern and Zhou(2020)]{austern2020asymptotics}
Morgane Austern and Wenda Zhou.
\newblock Asymptotics of cross-validation.
\newblock \emph{arXiv preprint arXiv:2001.11111}, 2020.

\bibitem[Azriel et~al.(2022)Azriel, Brown, Sklar, Berk, Buja, and
  Zhao]{azriel2022semi}
David Azriel, Lawrence~D Brown, Michael Sklar, Richard Berk, Andreas Buja, and
  Linda Zhao.
\newblock Semi-supervised linear regression.
\newblock \emph{Journal of the American Statistical Association}, 117\penalty0
  (540):\penalty0 2238--2251, 2022.

\bibitem[Ball et~al.(2017)Ball, Anderson, and Chan]{ball2017comprehensive}
John~E Ball, Derek~T Anderson, and Chee~Seng Chan.
\newblock Comprehensive survey of deep learning in remote sensing: theories,
  tools, and challenges for the community.
\newblock \emph{Journal of applied remote sensing}, 11\penalty0 (4):\penalty0
  042609--042609, 2017.

\bibitem[Bang and Robins(2005)]{bang2005doubly}
Heejung Bang and James~M Robins.
\newblock Doubly robust estimation in missing data and causal inference models.
\newblock \emph{Biometrics}, 61\penalty0 (4):\penalty0 962--973, 2005.

\bibitem[Bates et~al.(2023)Bates, Hastie, and Tibshirani]{bates2023cross}
Stephen Bates, Trevor Hastie, and Robert Tibshirani.
\newblock Cross-validation: what does it estimate and how well does it do it?
\newblock \emph{Journal of the American Statistical Association}, pages 1--12,
  2023.

\bibitem[Bayle et~al.(2020)Bayle, Bayle, Janson, and Mackey]{bayle2020cross}
Pierre Bayle, Alexandre Bayle, Lucas Janson, and Lester Mackey.
\newblock Cross-validation confidence intervals for test error.
\newblock \emph{Advances in Neural Information Processing Systems},
  33:\penalty0 16339--16350, 2020.

\bibitem[Bickel et~al.(1993)Bickel, Klaassen, Bickel, Ritov, Klaassen, Wellner,
  and Ritov]{bickel1993efficient}
Peter~J Bickel, Chris~AJ Klaassen, Peter~J Bickel, Ya’acov Ritov, J~Klaassen,
  Jon~A Wellner, and YA'Acov Ritov.
\newblock \emph{Efficient and adaptive estimation for semiparametric models},
  volume~4.
\newblock Springer, 1993.

\bibitem[Bludau et~al.(2022)Bludau, Willems, Zeng, Strauss, Hansen, Tanzer,
  Karayel, Schulman, and Mann]{bludau2022structural}
Isabell Bludau, Sander Willems, Wen-Feng Zeng, Maximilian~T Strauss, Fynn~M
  Hansen, Maria~C Tanzer, Ozge Karayel, Brenda~A Schulman, and Matthias Mann.
\newblock The structural context of posttranslational modifications at a
  proteome-wide scale.
\newblock \emph{PLoS biology}, 20\penalty0 (5):\penalty0 e3001636, 2022.

\bibitem[Buja et~al.(2019{\natexlab{a}})Buja, Brown, Berk, George, Pitkin,
  Traskin, Zhang, and Zhao]{buja2019modelsi}
Andreas Buja, Lawrence Brown, Richard Berk, Edward George, Emil Pitkin, Mikhail
  Traskin, Kai Zhang, and Linda Zhao.
\newblock Models as approximations {I}.
\newblock \emph{Statistical Science}, 34\penalty0 (4):\penalty0 523--544,
  2019{\natexlab{a}}.

\bibitem[Buja et~al.(2019{\natexlab{b}})Buja, Brown, Kuchibhotla, Berk, George,
  and Zhao]{buja2019modelsii}
Andreas Buja, Lawrence Brown, Arun~Kumar Kuchibhotla, Richard Berk, Edward
  George, and Linda Zhao.
\newblock Models as approximations {II}.
\newblock \emph{Statistical Science}, 34\penalty0 (4):\penalty0 545--565,
  2019{\natexlab{b}}.

\bibitem[Bullock et~al.(2020)Bullock, Woodcock, Souza~Jr, and
  Olofsson]{bullock2020satellite}
Eric~L Bullock, Curtis~E Woodcock, Carlos Souza~Jr, and Pontus Olofsson.
\newblock Satellite-based estimates reveal widespread forest degradation in the
  {A}mazon.
\newblock \emph{Global Change Biology}, 26\penalty0 (5):\penalty0 2956--2969,
  2020.

\bibitem[Chakrabortty et~al.(2022)Chakrabortty, Dai, and
  Carroll]{chakrabortty2022semi}
Abhishek Chakrabortty, Guorong Dai, and Raymond~J Carroll.
\newblock Semi-supervised quantile estimation: Robust and efficient inference
  in high dimensional settings.
\newblock \emph{arXiv preprint arXiv:2201.10208}, 2022.

\bibitem[Chen and Guestrin(2016)]{chen2016xgboost}
Tianqi Chen and Carlos Guestrin.
\newblock {XGBoost}: A scalable tree boosting system.
\newblock In \emph{Proceedings of the 22nd acm sigkdd international conference
  on knowledge discovery and data mining}, pages 785--794, 2016.

\bibitem[Chernozhukov et~al.(2018)Chernozhukov, Chetverikov, Demirer, Duflo,
  Hansen, Newey, and Robins]{chernozhukov2018double}
Victor Chernozhukov, Denis Chetverikov, Mert Demirer, Esther Duflo, Christian
  Hansen, Whitney Newey, and James Robins.
\newblock Double/debiased machine learning for treatment and structural
  parameters, 2018.

\bibitem[Chernozhukov et~al.(2022)Chernozhukov, Escanciano, Ichimura, Newey,
  and Robins]{chernozhukov2022locally}
Victor Chernozhukov, Juan~Carlos Escanciano, Hidehiko Ichimura, Whitney~K
  Newey, and James~M Robins.
\newblock Locally robust semiparametric estimation.
\newblock \emph{Econometrica}, 90\penalty0 (4):\penalty0 1501--1535, 2022.

\bibitem[Chernozhukov et~al.(2023)Chernozhukov, Newey, and
  Singh]{chernozhukov2023simple}
Victor Chernozhukov, Whitney~K Newey, and Rahul Singh.
\newblock A simple and general debiased machine learning theorem with
  finite-sample guarantees.
\newblock \emph{Biometrika}, 110\penalty0 (1):\penalty0 257--264, 2023.

\bibitem[Ding et~al.(2021)Ding, Hardt, Miller, and Schmidt]{ding2021retiring}
Frances Ding, Moritz Hardt, John Miller, and Ludwig Schmidt.
\newblock Retiring adult: New datasets for fair machine learning.
\newblock \emph{Advances in neural information processing systems},
  34:\penalty0 6478--6490, 2021.

\bibitem[Dudoit and van~der Laan(2005)]{dudoit2005asymptotics}
Sandrine Dudoit and Mark~J van~der Laan.
\newblock Asymptotics of cross-validated risk estimation in estimator selection
  and performance assessment.
\newblock \emph{Statistical methodology}, 2\penalty0 (2):\penalty0 131--154,
  2005.

\bibitem[Hansen et~al.(2013)Hansen, Potapov, Moore, Hancher, Turubanova,
  Tyukavina, Thau, Stehman, Goetz, Loveland, et~al.]{hansen2013high}
Matthew~C Hansen, Peter~V Potapov, Rebecca Moore, Matt Hancher, Svetlana~A
  Turubanova, Alexandra Tyukavina, David Thau, Stephen~V Stehman, Scott~J
  Goetz, Thomas~R Loveland, et~al.
\newblock High-resolution global maps of 21st-century forest cover change.
\newblock \emph{science}, 342\penalty0 (6160):\penalty0 850--853, 2013.

\bibitem[Hasminskii and Ibragimov(1979)]{hasminskii1979nonparametric}
Rafail~Z Hasminskii and Ildar~A Ibragimov.
\newblock On the nonparametric estimation of functionals.
\newblock In \emph{Proceedings of the Second Prague Symposium on Asymptotic
  Statistics}, volume 473, pages 474--482. North-Holland Amsterdam, 1979.

\bibitem[Jean et~al.(2016)Jean, Burke, Xie, Davis, Lobell, and
  Ermon]{jean2016combining}
Neal Jean, Marshall Burke, Michael Xie, W~Matthew Davis, David~B Lobell, and
  Stefano Ermon.
\newblock Combining satellite imagery and machine learning to predict poverty.
\newblock \emph{Science}, 353\penalty0 (6301):\penalty0 790--794, 2016.

\bibitem[Jin and Rothenh{\"a}usler(2023)]{jin2023tailored}
Ying Jin and Dominik Rothenh{\"a}usler.
\newblock Tailored inference for finite populations: conditional validity and
  transfer across distributions.
\newblock \emph{Biometrika}, page asad022, 2023.

\bibitem[Jumper et~al.(2021)Jumper, Evans, Pritzel, Green, Figurnov,
  Ronneberger, Tunyasuvunakool, Bates, {\v{Z}}{\'\i}dek, Potapenko,
  et~al.]{jumper2021highly}
John Jumper, Richard Evans, Alexander Pritzel, Tim Green, Michael Figurnov,
  Olaf Ronneberger, Kathryn Tunyasuvunakool, Russ Bates, Augustin
  {\v{Z}}{\'\i}dek, Anna Potapenko, et~al.
\newblock Highly accurate protein structure prediction with alphafold.
\newblock \emph{Nature}, 596\penalty0 (7873):\penalty0 583--589, 2021.

\bibitem[Kissel and Lei(2022)]{kissel2022high}
Nicholas Kissel and Jing Lei.
\newblock On high-dimensional {G}aussian comparisons for cross-validation.
\newblock \emph{arXiv preprint arXiv:2211.04958}, 2022.

\bibitem[Klaassen(1987)]{klaassen1987consistent}
Chris~AJ Klaassen.
\newblock Consistent estimation of the influence function of locally
  asymptotically linear estimators.
\newblock \emph{The Annals of Statistics}, 15\penalty0 (4):\penalty0
  1548--1562, 1987.

\bibitem[Lehmann and Romano(2022)]{lehmann2022testing}
Erich~L Lehmann and Joseph~P Romano.
\newblock \emph{Testing statistical hypotheses}, volume~4.
\newblock Springer Nature, 2022.

\bibitem[Levit(1976)]{levit1976efficiency}
B~Ya Levit.
\newblock On the efficiency of a class of non-parametric estimates.
\newblock \emph{Theory of Probability \& Its Applications}, 20\penalty0
  (4):\penalty0 723--740, 1976.

\bibitem[Lin et~al.(2023)Lin, Akin, Rao, Hie, Zhu, Lu, Smetanin, Verkuil,
  Kabeli, Shmueli, et~al.]{lin2023evolutionary}
Zeming Lin, Halil Akin, Roshan Rao, Brian Hie, Zhongkai Zhu, Wenting Lu, Nikita
  Smetanin, Robert Verkuil, Ori Kabeli, Yaniv Shmueli, et~al.
\newblock Evolutionary-scale prediction of atomic-level protein structure with
  a language model.
\newblock \emph{Science}, 379\penalty0 (6637):\penalty0 1123--1130, 2023.

\bibitem[Motwani and Witten(2023)]{motwani2023valid}
Keshav Motwani and Daniela Witten.
\newblock Valid inference after prediction.
\newblock \emph{arXiv preprint arXiv:2306.13746}, 2023.

\bibitem[Newey(1994)]{newey1994asymptotic}
Whitney~K Newey.
\newblock The asymptotic variance of semiparametric estimators.
\newblock \emph{Econometrica: Journal of the Econometric Society}, pages
  1349--1382, 1994.

\bibitem[Newey and McFadden(1994)]{newey1994large}
Whitney~K Newey and Daniel McFadden.
\newblock Large sample estimation and hypothesis testing.
\newblock \emph{Handbook of econometrics}, 4:\penalty0 2111--2245, 1994.

\bibitem[Robins and Rotnitzky(1995)]{robins1995semiparametric}
James~M Robins and Andrea Rotnitzky.
\newblock Semiparametric efficiency in multivariate regression models with
  missing data.
\newblock \emph{Journal of the American Statistical Association}, 90\penalty0
  (429):\penalty0 122--129, 1995.

\bibitem[Robins et~al.(1994)Robins, Rotnitzky, and Zhao]{robins1994estimation}
James~M Robins, Andrea Rotnitzky, and Lue~Ping Zhao.
\newblock Estimation of regression coefficients when some regressors are not
  always observed.
\newblock \emph{Journal of the American statistical Association}, 89\penalty0
  (427):\penalty0 846--866, 1994.

\bibitem[Robinson et~al.(2017)Robinson, Hohman, and Dilkina]{robinson2017deep}
Caleb Robinson, Fred Hohman, and Bistra Dilkina.
\newblock A deep learning approach for population estimation from satellite
  imagery.
\newblock In \emph{Proceedings of the 1st ACM SIGSPATIAL Workshop on Geospatial
  Humanities}, pages 47--54, 2017.

\bibitem[Robinson(1988)]{robinson1988root}
Peter~M Robinson.
\newblock Root-n-consistent semiparametric regression.
\newblock \emph{Econometrica: Journal of the Econometric Society}, pages
  931--954, 1988.

\bibitem[Rolf et~al.(2021)Rolf, Proctor, Carleton, Bolliger, Shankar, Ishihara,
  Recht, and Hsiang]{rolf2021generalizable}
Esther Rolf, Jonathan Proctor, Tamma Carleton, Ian Bolliger, Vaishaal Shankar,
  Miyabi Ishihara, Benjamin Recht, and Solomon Hsiang.
\newblock A generalizable and accessible approach to machine learning with
  global satellite imagery.
\newblock \emph{Nature communications}, 12\penalty0 (1):\penalty0 4392, 2021.

\bibitem[Rotnitzky et~al.(1998)Rotnitzky, Robins, and
  Scharfstein]{rotnitzky1998semiparametric}
Andrea Rotnitzky, James~M Robins, and Daniel~O Scharfstein.
\newblock Semiparametric regression for repeated outcomes with nonignorable
  nonresponse.
\newblock \emph{Journal of the american statistical association}, 93\penalty0
  (444):\penalty0 1321--1339, 1998.

\bibitem[Rubin(1987)]{rubin1987multiple}
D~Rubin.
\newblock Multiple imputation for nonresponse in surveys.
\newblock \emph{Wiley Series in Probability and Statistics}, page~1, 1987.

\bibitem[Rubin(1976)]{rubin1976inference}
Donald~B Rubin.
\newblock Inference and missing data.
\newblock \emph{Biometrika}, 63\penalty0 (3):\penalty0 581--592, 1976.

\bibitem[Rubin(1996)]{rubin1996multiple}
Donald~B Rubin.
\newblock Multiple imputation after 18+ years.
\newblock \emph{Journal of the American statistical Association}, 91\penalty0
  (434):\penalty0 473--489, 1996.

\bibitem[Schafer(1999)]{schafer1999multiple}
Joseph~L Schafer.
\newblock Multiple imputation: a primer.
\newblock \emph{Statistical methods in medical research}, 8\penalty0
  (1):\penalty0 3--15, 1999.

\bibitem[Sexton et~al.(2013)Sexton, Song, Feng, Noojipady, Anand, Huang, Kim,
  Collins, Channan, DiMiceli, et~al.]{sexton2013global}
Joseph~O Sexton, Xiao-Peng Song, Min Feng, Praveen Noojipady, Anupam Anand,
  Chengquan Huang, Do-Hyung Kim, Kathrine~M Collins, Saurabh Channan, Charlene
  DiMiceli, et~al.
\newblock Global, 30-m resolution continuous fields of tree cover:
  Landsat-based rescaling of modis vegetation continuous fields with
  lidar-based estimates of error.
\newblock \emph{International Journal of Digital Earth}, 6\penalty0
  (5):\penalty0 427--448, 2013.

\bibitem[Song et~al.(2023)Song, Lin, and Zhou]{song2023general}
Shanshan Song, Yuanyuan Lin, and Yong Zhou.
\newblock A general m-estimation theory in semi-supervised framework.
\newblock \emph{Journal of the American Statistical Association}, pages 1--11,
  2023.

\bibitem[Steele et~al.(2017)Steele, Sunds{\o}y, Pezzulo, Alegana, Bird,
  Blumenstock, Bjelland, Eng{\o}-Monsen, De~Montjoye, Iqbal,
  et~al.]{steele2017mapping}
Jessica~E Steele, P{\aa}l~Roe Sunds{\o}y, Carla Pezzulo, Victor~A Alegana,
  Tomas~J Bird, Joshua Blumenstock, Johannes Bjelland, Kenth Eng{\o}-Monsen,
  Yves-Alexandre De~Montjoye, Asif~M Iqbal, et~al.
\newblock Mapping poverty using mobile phone and satellite data.
\newblock \emph{Journal of The Royal Society Interface}, 14\penalty0
  (127):\penalty0 20160690, 2017.

\bibitem[Tony~Cai and Guo(2020)]{tony2020semisupervised}
T~Tony~Cai and Zijian Guo.
\newblock Semisupervised inference for explained variance in high dimensional
  linear regression and its applications.
\newblock \emph{Journal of the Royal Statistical Society Series B: Statistical
  Methodology}, 82\penalty0 (2):\penalty0 391--419, 2020.

\bibitem[Tunyasuvunakool et~al.(2021)Tunyasuvunakool, Adler, Wu, Green,
  Zielinski, {\v{Z}}{\'\i}dek, Bridgland, Cowie, Meyer, Laydon,
  et~al.]{tunyasuvunakool2021highly}
Kathryn Tunyasuvunakool, Jonas Adler, Zachary Wu, Tim Green, Michal Zielinski,
  Augustin {\v{Z}}{\'\i}dek, Alex Bridgland, Andrew Cowie, Clemens Meyer, Agata
  Laydon, et~al.
\newblock Highly accurate protein structure prediction for the human proteome.
\newblock \emph{Nature}, 596\penalty0 (7873):\penalty0 590--596, 2021.

\bibitem[van~der Vaart(1998)]{vandervaart1998asymptotic}
Aad~W. van~der Vaart.
\newblock \emph{Asymptotic Statistics}.
\newblock Cambridge Series in Statistical and Probabilistic Mathematics.
  Cambridge University Press, 1998.

\bibitem[Van~Engelen and Hoos(2020)]{van2020survey}
Jesper~E Van~Engelen and Holger~H Hoos.
\newblock A survey on semi-supervised learning.
\newblock \emph{Machine learning}, 109\penalty0 (2):\penalty0 373--440, 2020.

\bibitem[Wang et~al.(2020)Wang, McCormick, and Leek]{wang2020methods}
Siruo Wang, Tyler~H McCormick, and Jeffrey~T Leek.
\newblock Methods for correcting inference based on outcomes predicted by
  machine learning.
\newblock \emph{Proceedings of the National Academy of Sciences}, 117\penalty0
  (48):\penalty0 30266--30275, 2020.

\bibitem[White(1980)]{white1980heteroskedasticity}
Halbert White.
\newblock A heteroskedasticity-consistent covariance matrix estimator and a
  direct test for heteroskedasticity.
\newblock \emph{Econometrica: journal of the Econometric Society}, pages
  817--838, 1980.

\bibitem[White(1981)]{white1981consequences}
Halbert White.
\newblock Consequences and detection of misspecified nonlinear regression
  models.
\newblock \emph{Journal of the American Statistical Association}, 76\penalty0
  (374):\penalty0 419--433, 1981.

\bibitem[Willett et~al.(2013)Willett, Lintott, Bamford, Masters, Simmons,
  Casteels, Edmondson, Fortson, Kaviraj, Keel, et~al.]{willett2013galaxy}
Kyle~W Willett, Chris~J Lintott, Steven~P Bamford, Karen~L Masters, Brooke~D
  Simmons, Kevin~RV Casteels, Edward~M Edmondson, Lucy~F Fortson, Sugata
  Kaviraj, William~C Keel, et~al.
\newblock Galaxy zoo 2: detailed morphological classifications for 304 122
  galaxies from the {Sloan Digital Sky Survey}.
\newblock \emph{Monthly Notices of the Royal Astronomical Society},
  435\penalty0 (4):\penalty0 2835--2860, 2013.

\bibitem[York et~al.(2000)York, Adelman, Anderson~Jr, Anderson, Annis, Bahcall,
  Bakken, Barkhouser, Bastian, Berman, et~al.]{york2000sloan}
Donald~G York, J~Adelman, John~E Anderson~Jr, Scott~F Anderson, James Annis,
  Neta~A Bahcall, JA~Bakken, Robert Barkhouser, Steven Bastian, Eileen Berman,
  et~al.
\newblock The {Sloan Digital Sky Survey}: Technical summary.
\newblock \emph{The Astronomical Journal}, 120\penalty0 (3):\penalty0 1579,
  2000.

\bibitem[Zhang et~al.(2019)Zhang, Brown, and Cai]{zhang2019semi}
Anru Zhang, Lawrence~D Brown, and T~Tony Cai.
\newblock Semi-supervised inference: General theory and estimation of means.
\newblock \emph{Annals of Statistics}, 47\penalty0 (5):\penalty0 2538--2566,
  2019.

\bibitem[Zhang and Bradic(2022)]{zhang2022high}
Yuqian Zhang and Jelena Bradic.
\newblock High-dimensional semi-supervised learning: in search of optimal
  inference of the mean.
\newblock \emph{Biometrika}, 109\penalty0 (2):\penalty0 387--403, 2022.

\bibitem[Zheng et~al.(2023)Zheng, Zhang, Borgs, Chayes, and
  Yaghi]{chatgptchemistry}
Zhiling Zheng, Oufan Zhang, Christian Borgs, Jennifer~T. Chayes, and Omar~M.
  Yaghi.
\newblock Chatgpt chemistry assistant for text mining and the prediction of mof
  synthesis.
\newblock \emph{Journal of the American Chemical Society}, 145\penalty0
  (32):\penalty0 18048--18062, 2023.

\bibitem[Zhu and Goldberg(2022)]{zhu2022introduction}
Xiaojin Zhu and Andrew~B Goldberg.
\newblock \emph{Introduction to semi-supervised learning}.
\newblock Springer Nature, 2022.

\bibitem[Zhu(2005)]{zhu2005semi}
Xiaojin~Jerry Zhu.
\newblock Semi-supervised learning literature survey.
\newblock 2005.

\end{thebibliography}

\newpage

\appendix

\section{Proofs}

\subsection{Proof of Theorem \ref{thm:clt_mean} (CLT for mean estimation)}

The proof builds on two key technical lemmas, which leverage the notion of stability in Assumption~\ref{ass:stability} to show that we can ``replace'' the models $f^{(j)}$ in the definition of the cross-prediction estimator \eqref{eq:crossfit_mean} with the ``average'' model $\bar f$. Since $\bar f$ is a nonrandom model, we can proceed with a standard CLT analysis of the two terms comprising the estimator.

We begin by stating and proving the technical lemmas, which are inspired by the analysis of cross-validation due to~\citet{bayle2020cross}. To simplify notation, we use $\E_X$ and $\Var_X$ to denote the expectation and variance conditional on everything but $X$.

\begin{lemma}
\label{lemma:term1}
Suppose that the predictions are stable (Ass. \ref{ass:stability}). Denote
\begin{equation}
\label{eq:til_Fj}\tilde F_j = \frac{1}{\sqrt{N}} \sum_{i=1}^N \left(f^{(j)}(\Xt_i) - \E_{X} [f^{(j)}(X) ] - (\bar f(\Xt_i) - \E [\bar f(X) ])\right) .
\end{equation}
Then,
$$\frac 1 K \sum_{j=1}^K \tilde F_j \stackrel{p}{\to} 0.$$
\end{lemma}

\begin{proof}
Let $\psi(x) = \min(1,x)$. We will use the fact that $\frac 1 K \sum_{j=1}^K \tilde F_j \stackrel{p}{\to} 0$ if and only if $\E[\psi(|\frac 1 K \sum_{j=1}^K \tilde F_j|)]\rightarrow 0$, as stated in Fact \ref{fact:psi_equivalence} below. See, for example, \citet{bayle2020cross} for a proof of the fact.

\begin{fact}
\label{fact:psi_equivalence}
Let $X_n$ be a sequence of random variables. Then, $X_n\stackrel{p}{\to} 0$ if and only if $\E[\min(1,|X_n|)]\to 0$.
\end{fact}

Note that $\psi(x)$ is nondecreasing and satisfies $\psi\left(\sum_{j=1}^K x_j \right)\leq \sum_{j=1}^K \psi(x_j)$ for non-negative $x_j$; this yields
\begin{align*}
\E\left[\psi\left(\left| \frac 1 K \sum_{j=1}^K \tilde F_j \right|\right) \right] &\leq  \E\left[\psi\left( \frac 1 K \sum_{j=1}^K  | \tilde F_j |\right) \right] \leq  \sum_{j=1}^K \E\left[\psi\left( \frac 1 K   | \tilde F_j |\right) \right].
\end{align*}

Notice that $\psi(x)$ is also concave. Therefore, by Jensen's inequality, we have
\begin{align*}
\sum_{j=1}^K \E\left[\psi\left( \frac 1 K   | \tilde F_j |\right) \right]
&\leq \sum_{j=1}^K \E\left[\psi\left(  \frac 1 K  \E\left[|\tilde F_j| \Big| f^{(j)} \right] \right) \right]\\
&\leq \sum_{j=1}^K \E\left[ \psi\left( \frac 1 K  \sqrt{\E [\tilde F_j^2|f^{(j)}  ]}\right) \right]\\
&= \sum_{j=1}^K  \E\left[\psi\left( \frac 1 K \sqrt{\Var(\tilde F_j|f^{(j)} ) } \right) \right]  \\
&= \sum_{j=1}^K \E\left[ \psi\left( \frac 1 K \sqrt{\Var_{X}\left( f^{(j)}(X) - \bar f(X)\right) } \right) \right]\\
&= \E\left[\min\left(K, \sqrt{\Var_{X}\left( f^{(1)}(X) - \bar f(X)\right) } \right)\right].
\end{align*}
Invoking the stability condition shows that the right-hand side converges to zero. Note that, technically, the stability condition is stronger than what is needed for the expression above to converge to zero. In particular, stability ensures that $\E\left[\sqrt{K \cdot \Var_{X}\left( f^{(1)}(X) - \bar f(X)\right) }\right]\to 0$, while for the expression above to vanish it would suffice to ensure $\E\left[\sqrt{ \Var_{X}\left( f^{(1)}(X) - \bar f(X)\right) }\right]\to 0$. The stronger condition will be used in the next technical lemma, which handles the second term in the cross-prediction estimator.

Putting everything together, we get that $\frac 1 K \sum_{j=1}^K \tilde F_j\stackrel{p}{\to} 0 $, as desired.
\end{proof}

\begin{lemma}
\label{lemma:term2}
Suppose that the predictions are stable (Ass. \ref{ass:stability}). Denote
\begin{equation}
\label{eq:Fj}
F_j = \frac{1}{\sqrt{n}} \sum_{i\in I_j} \left(f^{(j)}(X_i) - \E_X[f^{(j)}(X)] - (\bar f(X_i) - \E [\bar f(X) ])\right).
\end{equation}
Then,
$$\sum_{j=1}^K F_j \stackrel{p}{\to} 0.$$
\end{lemma}

\begin{proof}
The proof follows a similar principle as the proof of Lemma \ref{lemma:term1}. 
As before, we let $\psi(x) = \min(1,x)$ and use the fact $\sum_{j=1}^K F_j\stackrel{p}{\to} 0$ if and only if $\E[\psi(|\sum_{j=1}^K F_j|)] \to 0$ (see Fact \ref{fact:psi_equivalence}). We use the fact that $\psi\left(\sum_{j=1}^K x_j \right)\leq \sum_{j=1}^K \psi(x_j)$ for non-negative $x_j$; this yields
\begin{align*}
\E\left[\psi\left(\left| \sum_{j=1}^K F_j \right|\right) \right] &\leq \E\left[\psi\left( \sum_{j=1}^K |F_j|\right) \right] \leq \sum_{j=1}^K \E\left[\psi\left(  |F_j|\right) \right].
\end{align*}
Next, by Jensen's inequality, we have
\begin{align*}
\sum_{j=1}^K \E\left[\psi\left(  |F_j|\right) \right]
&\leq \sum_{j=1}^K \E\left[\psi\left(   \E[|F_j|~| f^{(j)}] \right) \right]\\
&\leq \sum_{j=1}^K \E\left[\psi\left(  \sqrt{\E[F_j^2 | f^{(j)}] } \right) \right]\\
&= \sum_{j=1}^K \E\left[\psi\left(  \sqrt{\Var(F_j | f^{(j)}) } \right) \right]\\
&= \sum_{j=1}^K \E\left[\psi\left(  \sqrt{\frac{|I_j|}{n}\Var_{X} \left(f^{(j)}(X) - \bar f(X)\right) } \right) \right]\\
&= \E\left[\min\left(K, \sqrt{K\cdot \Var_{X} \left(f^{(1)}(X) - \bar f(X)\right) }\right)\right].
\end{align*}
Invoking the stability condition shows that the right-hand side converges to zero. Hence,  $ \sum_{j=1}^K F_j\stackrel{p}{\to} 0 $.
\end{proof}

With Lemma \ref{lemma:term1} and Lemma \ref{lemma:term2} in hand, we can now prove Theorem \ref{thm:clt_mean}. As alluded to earlier, the idea is to use Lemma \ref{lemma:term1} and Lemma \ref{lemma:term2} to replace the models $f^{(j)}$ in the definition of the cross-prediction estimator \eqref{eq:crossfit_mean}.

Writing $\theta^* = \frac{1}{K} \sum_{j=1}^K (\E[f^{(j)}(X)] - \E[f^{(j)}(X) - Y])$, we have
\begin{align}
\label{eq:T1_plus_T2}
&\frac{\sqrt{n}}{\sqrt{\frac n N \bar\sigma^2 + \bar \sigma_\Delta^2}} \left(\hat \theta^+ - \theta^* \right)\nonumber \\
&\quad = \frac{\sqrt{n}}{\sqrt{\frac n N \bar\sigma^2 + \bar \sigma_\Delta^2}} \left(\frac{1}{K  N} \sum_{j=1}^K \sum_{i=1}^N (f^{(j)}(\Xt_i) - \E_X[f^{(j)}(X)]) - \frac{1}{n} \sum_{j=1}^K \sum_{i\in I_j} ((f^{(j)}(X_i) - Y_i) - \E_{X,Y}[ f^{(j)}(X) - Y]) \right) \nonumber \\
&\quad = \frac{\sqrt{n}}{\sqrt{\frac n N \bar\sigma^2 + \bar \sigma_\Delta^2}} ( T_1 - T_2),
\end{align}
where we define
\begin{align*}
T_1 = \frac{1}{K  N} \sum_{j=1}^K \sum_{i=1}^N (f^{(j)}(\Xt_i) - \E_X[f^{(j)}(X)]);\quad 
T_2 =  \frac{1}{n} \sum_{j=1}^K \sum_{i\in I_j} ((f^{(j)}(X_i) - Y_i) - \E_{X,Y}[ f^{(j)}(X) - Y]).
\end{align*}
Focusing on $T_1$, we have
\begin{align*}
\sqrt{N} T_1 &= \frac 1 K \sum_{j=1}^K \tilde F_j + \frac{1}{\sqrt N} \sum_{i=1}^N  (\bar f(\Xt_i) - \E[\bar f(X)]),
\end{align*}
for $\tilde F_j$ defined in Eq.~\eqref{eq:til_Fj}. Invoking Lemma \ref{lemma:term1}, we thus get $\sqrt N T_1 = \frac{1}{\sqrt N} \sum_{i=1}^N  (\bar f(\Xt_i) - \E[\bar f(X)]) + o_P(1)$.

By an analogous argument, we have 
\begin{align*}
\sqrt{n} T_2 &= \sum_{j=1}^K F_j + \frac{ 1}{\sqrt{n}} \sum_{j=1}^K \sum_{i\in I_j} \left((\bar f(X_i) - Y_i) - \E_{X,Y}[\bar f(X) - Y]\right)\\
&= \frac{ 1}{\sqrt{n}} \sum_{i=1}^n  \left((\bar f(X_i) - Y_i) - \E_{X,Y}[\bar f(X) - Y]\right) + o_P(1),
\end{align*}
for $F_j$ defined in Eq.~\eqref{eq:Fj}. Going back to Eq.~\eqref{eq:T1_plus_T2} and denoting by $r, \bar \sigma^2_{\lim}, \bar\sigma_{\Delta,\lim}^2$ the limits of $\frac n N, \bar \sigma^2, \bar \sigma^2_\Delta$, respectively, we get
\begin{align*}
&\frac{\sqrt{n}}{\sqrt{\frac n N \bar\sigma^2 + \bar \sigma_\Delta^2}} \left(\hat \theta^+ - \theta^* \right)\\
&\quad = \frac{\sqrt{\frac n N}}{\sqrt{\frac n N \bar\sigma^2 + \bar \sigma_\Delta^2}} \sqrt{N} T_1 - \frac{1}{\sqrt{\frac n N \bar\sigma^2 + \bar \sigma_\Delta^2}} \sqrt{n} T_2\\
&\quad = \frac{\sqrt{\frac n N}}{\sqrt{\frac n N \bar\sigma^2 + \bar \sigma_\Delta^2}} \frac{1}{\sqrt N} \sum_{i=1}^N  (\bar f(\Xt_i) - \E[\bar f(X)]) - \frac{1}{\sqrt{\frac n N \bar\sigma^2 + \bar \sigma_\Delta^2}} \frac{ 1}{\sqrt{n}} \sum_{i=1}^n  \left((\bar f(X_i) - Y_i) - \E_{X,Y}[\bar f(X) - Y]\right) + o_P(1).
\end{align*}
By the Lindeberg central limit theorem, the first term converges in distribution to $\mathcal N\left(0, \frac{r \bar\sigma_{\lim}^2}{r \bar\sigma_{\lim}^2 + \bar \sigma_{\Delta, \lim}^2} \right)$, and the second term converges in distribution to $\mathcal N\left(0, \frac{\bar \sigma_{\Delta, \lim}^2}{r \bar\sigma_{\lim}^2 + \bar \sigma_{\Delta, \lim}^2} \right)$. Moreover, since the two terms are independent, we finally have
$$\frac{\sqrt{n}}{\sqrt{\frac n N \bar\sigma^2 + \bar \sigma_\Delta^2}} \left(\hat \theta^+ - \theta^* \right) \cd \mathcal N(0,1).$$

\subsection{Proof of Theorem \ref{thm:clt_general} (CLT for general M-estimation)}

The proof follows a similar template as the proof of Theorem \ref{thm:clt_mean}. We begin with two technical lemmas that allow swapping the models $f^{(j)}$ in the gradient of the cross-prediction loss,  $\nabla L^+(\theta)$, with the ``average'' model $\bar f$.

\begin{lemma}
\label{lemma:term1_general}
Suppose that the predictions are stable (Ass. \ref{ass:stability_general}). Denote
\begin{equation}
\label{eq:til_Lj}
\tilde L_j = \frac{1}{\sqrt{N}} \sum_{i=1}^N \left(\nabla \tilde \ell_{\theta,i}^{f^{(j)}} - \E [\nabla \ell_{\theta}^{f^{(j)}}] - (\nabla \tilde \ell_{\theta,i}^{\bar f} - \E [\nabla \ell_{\theta}^{\bar f}])\right) .
\end{equation}
Then,
$$\frac 1 K \sum_{j=1}^K \tilde L_j \stackrel{p}{\to} 0.$$
\end{lemma}

\begin{lemma}
\label{lemma:term2_general}
Suppose that the predictions are stable (Ass. \ref{ass:stability_general}). Denote
\begin{equation}
\label{eq:Lj}
L_j = \frac{1}{\sqrt{n}} \sum_{i\in I_j} \left(\nabla  \ell_{\theta,i}^{f^{(j)}} - \E[\nabla  \ell_{\theta}^{f^{(j)}}] - (\nabla  \ell_{\theta,i}^{\bar f} - \E[\nabla  \ell_{\theta}^{\bar f}])\right).
\end{equation}
Then,
$$\sum_{j=1}^K L_j \stackrel{p}{\to} 0.$$
\end{lemma}

Lemma \ref{lemma:term1_general} and Lemma \ref{lemma:term2_general} are proved completely analogously to Lemma \ref{lemma:term1} and Lemma \ref{lemma:term2}; we apply the same argument as before entry-wise.

Now we put the lemmas together to prove the central limit theorems. We first analyze the asymptotic normality of $\nabla L^+(\theta)$. The asymptotic normality of $\hat\theta^+$ relies on a similar application of the lemmas.

\paragraph{Asymptotic normality of $\nabla L^+(\theta)$.} We can write
\begin{align}
\label{eq:T1_T2_decomp}
&\sqrt{n} \left(\frac n N \bar\Sigma_\theta + \bar \Sigma_{\Delta,\theta} \right)^{-1/2} (\nabla L^+(\theta) - \nabla L(\theta)) \nonumber\\
&\quad = \sqrt{n} \left(\frac n N \bar\Sigma_\theta + \bar \Sigma_{\Delta,\theta} \right)^{-1/2} \left(\frac{1}{K N} \sum_{j=1}^K \sum_{i=1}^N  \nabla \tilde \ell_{\theta,i}^{f^{(j)}} - \frac{1}{n} \sum_{j=1}^K \sum_{i\in I_j}( \nabla \ell_{\theta,i}^{f^{(j)}} - \nabla \ell_{\theta,i} ) - \nabla L(\theta) \right) \nonumber \\
&\quad = \sqrt{n} \left(\frac n N \bar\Sigma_\theta + \bar \Sigma_{\Delta,\theta} \right)^{-1/2} \left(\frac{1}{K N} \sum_{j=1}^K \sum_{i=1}^N  (\nabla \tilde \ell_{\theta,i}^{f^{(j)}} - \E[\nabla \ell_{\theta}^{f^{(j)}}] ) - \frac{1}{n} \sum_{j=1}^K \sum_{i\in I_j}( (\nabla \ell_{\theta,i}^{f^{(j)}} - \nabla \ell_{\theta,i}) - \E[\nabla \ell_{\theta}^{f^{(j)}} - \nabla \ell_{\theta}]) \right) \nonumber \\
&\quad = \sqrt{n} \left(\frac n N \bar\Sigma_\theta + \bar \Sigma_{\Delta,\theta} \right)^{-1/2} \left(T_1 - T_2\right),
\end{align}
where we define
\begin{align*}
T_1 = \frac{1}{K N} \sum_{j=1}^K \sum_{i=1}^N  (\nabla \tilde \ell_{\theta,i}^{f^{(j)}} - \E[\nabla \ell_{\theta}^{f^{(j)}}] ); \quad
T_2 = \frac{1}{n} \sum_{j=1}^K \sum_{i\in I_j}( (\nabla \ell_{\theta,i}^{f^{(j)}} - \nabla \ell_{\theta,i}) - \E[\nabla \ell_{\theta}^{f^{(j)}} - \nabla \ell_{\theta}]).
\end{align*}
We apply Lemma \ref{lemma:term1_general} and Lemma \ref{lemma:term2_general} to $T_1$ and $T_2$, respectively. In particular, we can write
\begin{align*}
\sqrt{N} T_1 = \frac{1}{K} \sum_{j=1}^K \tilde L_j + \frac{1}{\sqrt{N}} \sum_{i=1}^N (\nabla \tilde \ell_{\theta,i}^{\bar f} - \E[\nabla \ell_{\theta}^{\bar f}]),
\end{align*}
where $\tilde L_j$ is given in Eq.~\eqref{eq:til_Lj}. Invoking Lemma \ref{lemma:term1_general}, we thus have $\sqrt N T_1 = \frac{1}{\sqrt{N}} \sum_{i=1}^N (\nabla \tilde \ell_{\theta,i}^{\bar f} - \E[\nabla \ell_{\theta}^{\bar f}]) + o_P(1)$.

By an analogous argument, for $L_j$ defined in Eq.~\eqref{eq:Lj}, we have
\begin{align*}
\sqrt{n} T_2 &= \sum_{j=1}^K  L_j + \frac{1}{\sqrt{n}} \sum_{j=1}^K \sum_{i\in I_j} \left((\nabla  \ell_{\theta,i}^{\bar f} - \nabla  \ell_{\theta,i}) - \E[\nabla \ell_{\theta}^{\bar f} - \nabla \ell_{\theta}]\right)\\
&= \frac{1}{\sqrt{n}} \sum_{i=1}^n \left((\nabla  \ell_{\theta,i}^{\bar f} - \nabla  \ell_{\theta,i}) - \E[\nabla \ell_{\theta}^{\bar f} - \nabla \ell_{\theta}]\right) + o_P(1).
\end{align*}

Going back to Eq.~\eqref{eq:T1_T2_decomp} and denoting by $r, \bar \Sigma_{\theta,\lim}, \bar \Sigma_{\Delta,\theta,\lim}$ the limits of $\frac n N, \bar \Sigma_{\theta}, \bar \Sigma_{\Delta,\theta}$, respectively, by the Lindeberg central limit theorem we get
\begin{align*}
\sqrt{n} \left(\frac n N \bar\Sigma_\theta + \bar \Sigma_{\Delta,\theta} \right)^{-1/2} \left(\nabla L^+(\theta) - \nabla L(\theta) \right)
&= \left(\frac n N \bar\Sigma_\theta + \bar \Sigma_{\Delta,\theta} \right)^{-1/2} \left(\frac{\sqrt{n}}{\sqrt N} \sqrt N T_1 - \sqrt{n} T_2\right)\\
&= \left(\frac n N \bar\Sigma_\theta + \bar \Sigma_{\Delta,\theta} \right)^{-1/2} \sqrt{\frac n N} \frac{1}{\sqrt{N}} \sum_{i=1}^N (\nabla \tilde \ell_{\theta,i}^{\bar f} - \E[\nabla \ell_{\theta}^{\bar f}])\\
&- \left(\frac n N \bar\Sigma_\theta + \bar \Sigma_{\Delta,\theta} \right)^{-1/2} \frac{1}{\sqrt{n}} \sum_{i=1}^n \left((\nabla  \ell_{\theta,i}^{\bar f} - \nabla  \ell_{\theta,i}) - \E[\nabla \ell_{\theta}^{\bar f}- \nabla \ell_{\theta}]\right)
\end{align*}
The first term above converges in distribution to 
$$\mathcal N\left(0, r\cdot (r \bar\Sigma_{\theta,\lim} + \bar \Sigma_{\Delta,\theta,\lim} )^{-1/2} \bar\Sigma_{\theta,\lim}  (r \bar\Sigma_{\theta,\lim} + \bar \Sigma_{\Delta,\theta,\lim} )^{-1/2} \right),$$
and the second term converges in distribution to 
$$\mathcal N\left(0, (r \bar\Sigma_{\theta,\lim} + \bar \Sigma_{\Delta,\theta,\lim} )^{-1/2} \bar \Sigma_{\Delta,\theta,\lim}  (r \bar\Sigma_{\theta,\lim} + \bar \Sigma_{\Delta,\theta,\lim} )^{-1/2} \right).$$
Since the two terms are independent, we can add up their limiting covariance matrices and get
$$\sqrt{n} \left(\frac n N \bar\Sigma_\theta + \bar \Sigma_{\Delta,\theta} \right)^{-1/2} \left(\nabla L^+(\theta) - \nabla L(\theta) \right) \cd \mathcal N(0,I),$$
as desired.

\paragraph{Asymptotic normality of $\hat\theta^+$.} 
We follow an argument similar to the classical proof of asymptotic normality of M-estimators (see Theorem 5.23 in \cite{vandervaart1998asymptotic}).
Given a function $g$, we use the shorthand notation
\begin{align}
  \E_n g &:= \frac{1}{n} \sum_{i=1}^n g(X_i,Y_i), \quad \GG_n g = \sqrt{n}\left(\E_n g - \E[g(X,Y)]\right); \label{eq:Gn}\\
  \widetilde \E_N^{f^+} g &:= \frac{1}{N K} \sum_{i=1}^N \sum_{j=1}^K g(\widetilde X_i, f^{(j)}(\widetilde X_i)),\quad \widetilde \GG_N^{f^+} g := \sqrt{N} \left( \widetilde \E_N^{f^+} g - \frac 1 K \sum_{j=1}^K \E_X[g(X,f^{(j)}(X)) ] \right); \label{eq:GNf}\\
  \widetilde \E_N^{\bar f} g &:= \frac{1}{N } \sum_{i=1}^N  g(\widetilde X_i, \bar f(\widetilde X_i)),\quad \widetilde \GG_N^{\bar f} g := \sqrt{N} \left( \widetilde \E_N^{\bar f} g - \E[g(X,\bar f(X))] \right); \label{eq:GNfbar}\\
  \E_n^{f^+} g &:= \frac{1}{n} \sum_{j=1}^K \sum_{i\in I_j} g( X_i, f^{(j)}(X_i)), \quad \GG_n^{f^+} g := \sqrt{n}  \left(\E_n^{f^+} g - \frac 1 K \sum_{j=1}^K \E_X[g(X,f^{(j)}(X))]\right); \label{eq:Gnf}\\
  \E_n^{\bar f} g &:= \frac{1}{n}  \sum_{i = 1}^n g( X_i, \bar f(X_i)), \quad \GG_n^{\bar f} g := \sqrt{n}  (\E_n^{\bar f} g -  \E[g(X,\bar f(X))]) \label{eq:Gnfbar }.
\end{align}

The differentiability and local Lipschitzness of the loss at $\theta^*$
imply that
for every (possibly random) sequence $h_n = O_P(1)$, we have
\begin{align*}
  \GG_n^{f^+}\left[\sqrt{n} \left(\ell_{\theta^* + \frac{h_n}{\sqrt{n}}} - \ell_{\theta^*} \right) \right] &= \GG_n^{f^+}\left[ h_n^\top \nabla \ell_{\theta^*}\right] + o_P(1).
\end{align*}
By essentially the same argument as in the proof of the asymptotic normality of $\nabla L^+(\theta)$, we can substitute the average over models $f^{(j)}$ for $\bar f$ via Lemma \ref{lemma:term2_general}, thus getting
\[ \GG_n^{f^+}\left[ h_n^\top \nabla \ell_{\theta^*}\right] = \GG_n^{\bar f}\left[ h_n^\top \nabla \ell_{\theta^*}\right] + o_P(1).\]
Analogously, we have $\GG_n [\sqrt{n} (\ell_{\theta^* + h_n/\sqrt{n}} - \ell_{\theta^*} )] = \GG_n[h_n^\top \nabla \ell_{\theta^*}] + o_P(1)$, and, by using Lemma~\ref{lemma:term1_general}, 
\[\tilde \GG_N^{f^+} [\sqrt{n} (\ell_{\theta^* + h_n/\sqrt{n}} - \ell_{\theta^*} )] = \tilde \GG_N^{\bar f}[h_n^\top \nabla \ell_{\theta^*}] + o_P(1).\]
Next, denoting $H_{\theta^*}^{f^+} = \frac 1 K \sum_{j=1}^K \nabla^2\E_X[\ell_{\theta^*}(X,f^{(j)}(X))]$, we apply a second-order Taylor expansion to get
\begin{align*}
  n \E_n\left(\ell_{\theta^* + \frac{h_n}{\sqrt{n}}} - \ell_{\theta^*}\right)
  &= \frac 1 2 h_n^\top H_{\theta^*} h_n
  + h_n^\top \GG_n \nabla\ell_{\theta^*} + o_P(1);\\
   n \widetilde \E_N^{ f^+}\left(\ell_{\theta^* + \frac{h_n}{\sqrt{n}}} - \ell_{\theta^*}\right) &= \frac{1}{2} h_n^\top H_{\theta^*}^{f^+} h_n + \sqrt{\frac{n}{N}} h_n ^\top \widetilde \GG_N^{\bar f} \nabla \ell_{\theta^*} + o_P(1);\\
   - n \E_n^{f^+}\left(\ell_{\theta^* + \frac{h_n}{\sqrt{n}}} - \ell_{\theta^*}\right) &=  -\frac{1}{2} h_n^\top H_{\theta^*}^{f^+} h_n - h_n ^\top \GG_n^{\bar f}  \nabla \ell_{\theta^*} + o_P(1).\\
\end{align*}
Notice that we have $L^+(\theta) = \tilde \E_N^{f^+} \ell_\theta + \E_n \ell_\theta - \E_n^{f^+} \ell_\theta$. Thus, if we add up the three equations above, we get the following:
$$n \left(L^+\left(\theta^* + \frac{h_n}{\sqrt{n}}\right) - L^+(\theta^*)\right) = \frac{1}{2} h_n^\top H_{\theta^*} h_n + h_n^\top \left(\GG_n \nabla\ell_{\theta^*} + \sqrt{\frac{n}{N}} \tilde \GG_N^{\bar f} \nabla\ell_{\theta^*} - \GG_n^{\bar f} \nabla\ell_{\theta^*}\right) + o_P(1).$$

We now evaluate this expression for two values of $h_n$, in particular $h_n^* = \sqrt{n}(\hat\theta^+ - \theta^*)$ and $h_n' = -H_{\theta^*}^{-1}\left(\GG_n \nabla\ell_{\theta^*} + \sqrt{\frac{n}{N}} \tilde \GG_N^{\bar f} \nabla\ell_{\theta^*} - \GG_n^{\bar f} \nabla\ell_{\theta^*}\right)$.
This gives
\begin{align*}
&n \left(L^+(\hat \theta^+) - L^+(\theta^*)\right) = \frac 1 2 (h_n^*)^\top H_{\theta^*} h_{n}^* + (h_n^*)^\top \left(\GG_n \nabla\ell_{\theta^*} + \sqrt{\frac{n}{N}} \tilde \GG_N^{\bar f} \nabla\ell_{\theta^*} - \GG_n^{\bar f} \nabla\ell_{\theta^*}\right) + o_P(1);\\
&n \left(L^+(\theta^* - h_n'/\sqrt{n}) - L^+(\theta^*)\right) = -\frac 1 2 (h_n')^\top H_{\theta^*} h_n' + o_P(1);
\end{align*}
By the definition of $\hat\theta^+$, the left-hand side of the first equation is smaller than the left-hand side of the second equation, hence the same relation is true for the right-hand sides. Taking the difference of the right-hand sides and completing the square gives
$$\frac 1 2 (h_n^* - h_n')^\top H_{\theta^*}(h_n^* - h_n') + o_P(1) \leq 0.$$
Since the Hessian is positive definite, we must have $h_n^* = h_n' + o_P(1)$, that is:
$$\sqrt n (\hat\theta^+ - \theta^*) =  - H_{\theta^*}^{-1}\left(\GG_n \nabla\ell_{\theta^*} + \sqrt{\frac{n}{N}} \tilde \GG_N^{\bar f} \nabla\ell_{\theta^*} - \GG_n^{\bar f} \nabla\ell_{\theta^*} \right) + o_P(1).$$
The final statement follows by a standard application of the central limit theorem to the second term. In particular, letting $r = \lim \frac n N$, we have:
\begin{align*}
    &\GG_n \nabla\ell_{\theta^*} + \sqrt{\frac{n}{N}} \tilde \GG_N^{\bar f} \nabla\ell_{\theta^*} - \GG_n^{\bar f} \nabla\ell_{\theta^*}\\
    &\quad = \frac{1}{\sqrt n} \sum_{i=1}^n \left( \nabla \ell_{\theta^*,i} - \nabla \ell^{\bar f}_{\theta^*,i} - \E[\nabla \ell_{\theta^*,i} - \nabla \ell^{\bar f}_{\theta^*,i}] \right) + \sqrt{\frac n N} \frac{1}{\sqrt{N}} \sum_{i=1}^N \left(\nabla \ell_{\theta^*,i}^{\bar f} - \E[\nabla \ell_{\theta^*}^{\bar f}] \right)\\
    &\quad \stackrel{d}{\to} \mathcal N\left(0, \Var(\nabla \ell_{\theta^*} - \nabla \ell_{\theta^*}^{\bar f} ) + r \Var(\nabla \ell_{\theta^*}^{\bar f})\right).
\end{align*}
Therefore, $- H_{\theta^*}^{-1}\left(\GG_n \nabla\ell_{\theta^*} + \sqrt{\frac{n}{N}} \tilde \GG_N^{\bar f} \nabla\ell_{\theta^*} - \GG_n^{\bar f} \nabla\ell_{\theta^*} \right)$ converges to $\mathcal{N}(0,\bar \Sigma)$, where 
$$\bar \Sigma = H_{\theta^*}^{-1} \left(\bar \Sigma_{\Delta,\theta^*}  + r \bar\Sigma_{\theta^*} \right) H_{\theta^*}^{-1}.$$

\end{document}